\documentclass[format=acmsmall, review=false]{acmart}

\usepackage{acm-ec-26}
\usepackage{booktabs}
\usepackage[ruled]{algorithm2e}

\SetAlFnt{\small}
\SetAlCapFnt{\small}
\SetAlCapNameFnt{\small}
\SetAlCapHSkip{0pt}
\IncMargin{-\parindent}

\usepackage{longtable}
\usepackage{newtxtext,newtxmath}
\usepackage{hyperref}
\usepackage{url}
\usepackage{amsfonts}
\usepackage{amsmath}
\usepackage{mathtools}
\usepackage{nicefrac}
\usepackage{microtype}
\usepackage{graphicx}
\usepackage{adjustbox}
\usepackage{algorithmic}
\usepackage{xcolor}
\usepackage{tcolorbox}
\usepackage{tikz}
\usepackage{tabularx}
\usepackage{ragged2e}
\usepackage{enumerate}
\usepackage{float}
\usepackage{soul}
\usepackage{verbatim}
\usepackage[capitalize,noabbrev]{cleveref}
\usetikzlibrary{shapes, arrows.meta, positioning}

\definecolor{aspRigor}{RGB}{31,119,180}
\definecolor{aspNovelty}{RGB}{44,160,44}
\definecolor{aspImpact}{RGB}{255,127,14}
\definecolor{aspEDR}{RGB}{148,103,189}
\definecolor{aspClarity}{RGB}{23,190,207}
\definecolor{aspRepr}{RGB}{140,86,75}
\definecolor{aspRW}{RGB}{188,189,34}
\definecolor{aspLim}{RGB}{214,39,40}
\definecolor{aspEthics}{RGB}{227,119,194}
\definecolor{aspFWP}{RGB}{127,127,127}
\newcommand{\hlaspect}[2]{\textcolor{#1!75!black}{#2}}

\newcommand{\norm}[1]{\left\lVert #1 \right\rVert}

\newcommand{\E}{\mathbb{E}}
\newcommand{\R}{\mathbb{R}}
\DeclareMathOperator*{\argmin}{argmin}
\DeclareMathOperator*{\argmax}{argmax}
\newcommand{\Var}{\operatorname{Var}}

\theoremstyle{definition}
\newtheorem{assumption}{Assumption}
\theoremstyle{plain}
\newtheorem{theorem}{Theorem}
\theoremstyle{remark}
\newtheorem{remark}{Remark}
\crefname{assumption}{Assumption}{Assumptions}
\Crefname{assumption}{Assumption}{Assumptions}
\crefname{remark}{Remark}{Remarks}

\tcbuselibrary{skins,breakable}
\newtcolorbox{aspectbox}[1][]{
  colback=blue!5!white,
  colframe=blue!75!black,
  fonttitle=\bfseries,
  title={Aspect Example},
  #1
}
\setcitestyle{authoryear}

\title{When Should Humans Step In? Optimal Human Dispatching in AI-Assisted Decisions}

\author{Lezhi Tan}
\affiliation{
  \institution{Stanford University}
  \city{Stanford}
  \country{USA}
}
\email{carrie13@stanford.edu}

\author{Naomi Sagan}
\affiliation{
  \institution{Stanford University}
  \city{Stanford}
  \country{USA}
}
\email{nsagan@stanford.edu}

\author{Lihua Lei}
\affiliation{
  \institution{Stanford University}
  \city{Stanford}
  \country{USA}
}
\email{lihualei@stanford.edu}

\author{Jos\'e Blanchet}
\affiliation{
  \institution{Stanford University}
  \city{Stanford}
  \country{USA}
}
\email{jose.blanchet@stanford.edu}

\begin{abstract}
AI systems increasingly assist decision making by producing cheap factor-level assessments of complex inputs, but these assessments can be biased or incomplete. We study \emph{selective human evaluation}: given AI signals, which factors should be dispatched to human experts? We formulate this human--AI collaboration problem as a factor-level information-acquisition problem. Under squared loss, the optimal dispatching rule maximizes a contextual reward; under a linear model, this reward admits a closed-form decomposition into two interpretable terms: predictive relevance and residual uncertainty in the human evaluation after conditioning on the AI signal.

We instantiate the framework on peer review, decomposing papers into ten aspects and evaluating up to 3,408 ICLR submissions across three LLMs and multiple regression heads. In this retrospective proxy, where human aspect signals are extracted from existing reviews, predictions that query three of the ten aspects under the optimal dispatching rules match the full ten-aspect human benchmark. The finding holds when the human signal comes from a single noisy reviewer, and across review years. These results position principled dispatching as a practical building block for scalable human--AI decision systems. We also discuss safeguards under which such a rule can serve as decision support that leaves the final decision with a human.

\end{abstract}

\begin{document}

\maketitle
% No ACM journal running header/footer ("Vol. 1, No. 1, ... Publication date") on the arXiv version.
\pagestyle{plain}

% Preserve the corrected grant information from the arXiv version.
\begingroup
\renewcommand\thefootnote{}\footnotetext{J. Blanchet gratefully acknowledges support from ONR under award N00014-24-1-2655, and the National Science Foundation (NSF) under grants 2312204 and 2403007. L.\ Lei is grateful for the support of National Science Foundation grant DMS-2338464. N.\ Sagan is supported by the Stanford Graduate Fellowship (SGF) for Sciences and Engineering, and by the National Science Foundation Graduate Research Fellowship Program under Grant No.\ DGE-2146755.}
\endgroup

% Updated paper content from the NeurIPS version.
% =============================================================================
%  sections/intro.tex
%  Target: ~1.5 pages.
%
%  Audience cue: NeurIPS "Human Interaction in AI Systems" reviewers care about
%  (a) why the human is in the loop at all, (b) what the human contributes that
%  the model cannot, and (c) whether the proposed mechanism is testable in real
%  deployment. Lead with all three; defer general framing.
% =============================================================================

\section{Introduction}
\label{sec:intro}

Many high-stakes decisions rest on a long, technical document --- a legal brief, a medical record, 
a financial filing, or a scientific manuscript --- that a human expert distills into a small number of 
\emph{factor-level} judgments before settling on a final decision. Peer review is a clean instance: 
reviewers assess papers along factors such as novelty, soundness, and clarity, and aggregate those into 
an overall score. Doing this well takes hours of expert effort per paper. Large language models can now produce factor-level assessments cheaply and at scale \citep{liang2024feedback,naddaf2025aipeerreview}, 
but their judgments are also documented to be systematically overly optimistic, weak on subtle technical claims, 
and biased toward surface writing quality~\citep{pangram2025,fanous2025syceval,shin2025blindspots}. 
Treating an AI system as a stand-in for a human expert is therefore neither safe nor honest; while forbidding the use of AI
 wastes its information value.

We propose a middle path: \emph{selective human evaluation} on top of AI-generated factor-level assessments. Suppose an AI produces preliminary evaluations for a set of $J$ pre-specified factors. A human is asked to review only a budgeted subset $S \subseteq [J]$, and a downstream predictor combines the always-available AI signals with the queried human signals to produce a prediction. The central question is then: which factors should be escalated, given the realized AI output, so as to optimize the final prediction? We call a rule that answers this question a \emph{dispatching rule}. In the language of prior work \citep{madras2018predict,mozannar2020consistent}, defer-or-predict frameworks decide whether to hand off an entire instance to a human; we study a factor-level analogue in which human effort is allocated within an instance. Throughout, the output of the pipeline is a recommendation with factor-level evidence for a human decision maker, who retains authority over the final decision (\Cref{sec:broader_impact}).

\paragraph{Our contributions.}
\begin{itemize}\itemsep=2pt

\item We introduce a general human--AI collaboration framework for decision-making problems that uses imperfect LLM outputs while allocating human effort selectively.
%, and formalize how human judgments can be used efficiently to correct AI outputs without requiring humans to generate all supporting reasons from scratch.
\item We formulate dispatching as an optimal selection problem conditioned on an always-available AI signal. Under squared loss and a mild well-specification assumption, we reduce the problem to maximizing a contextual reward (\Cref{thm:general_reduction}). 
Under a linear model, this reward decomposes into the \emph{predictive relevance} and the \emph{residual uncertainty} (\Cref{sec:linear_insight}). The resulting rules are optimal under these stated assumptions. 
\item We apply the framework to peer review. We decompose the task into 10 disjoint yet comprehensive aspects and construct both AI and human signals at the aspect level. Across 3,408 ICLR papers from 2019 and 2021, three LLMs (GPT-5-mini, GPT-4.1, and Gemini-2.0 Flash), and multiple regression models, the closed-form linear rule queries three aspects and comes within $0.013$ MAE of full ten-aspect human review. In this retrospective proxy, the results are preserved across years, under a cross-reviewer holdout that separates the text producing the human signal from the target, and for accept/reject prediction. On this dataset the rule recovers a fixed subset of aspects shared by all papers, while importance-only and disagreement-based rules are substantially worse; in a simulation with regime structure, adaptive rules pay off at realistic sample sizes once the estimator is centered (\Cref{app:simulation}).
\end{itemize}

The framework is application-agnostic, but we focus on peer review as a representative setting where both AI assistance and human effort are central. In this setting, decisions are naturally decomposed into factor-level judgments. The same structure arises in other domains involving complex documents, such as medical records, legal briefs, and financial filings, where decisions are based on a set of interpretable factors. We view peer review as a clean testbed that reflects this broader pattern.

The LLMs in our experiments are relatively low-cost, but the selection rule is model-agnostic. This choice is deliberate: we envision deployment at scale, where inexpensive AI signals determine which factors to escalate rather than making the final decision, so efficient models are the cost-effective first stage. Extending the same dispatching from humans to stronger reasoning models is an important direction for future work.

% We primarily tested on three low-cost LLMs  as the first-stage AI signal because the deployment pattern we target begins with a cheap, widely accessible model that processes all instances at scale, and escalates only a small subset of factors to human or a more expensive strong model (see further discussion in Section~\ref{sec:conclusion}).

\paragraph{Related work.}
Our problem sits at the intersection of three threads. 
1. \emph{Defer to human}. Learning-to-defer methods \citep{madras2018predict, mozannar2020consistent, wilder2020learning} route entire instances to humans, typically via a learned rejector or team-level objective; recent extensions incorporate LLM judges as surrogate rewards with occasional human queries \citep{ji2025multiarmedbanditsmachinelearninggenerated, ao2026bestarmidentificationllm}. 
Concept bottleneck models \citep{koh2020concept} allow humans to override intermediate predictions, with recent work on deferring CBMs learning when to defer specific concepts \citep{pugnana2025deferring}. They rely on learned surrogate objectives, whereas we derive a closed-form selection rule based on predictive relevance and residual uncertainty, which is directly estimable. \Cref{tab:related_work}  summarizes the comparison. 
2. \emph{Information acquisition}. Feature acquisition has been a long-studied problem \citep{saar2009active, xu2012greedy, janisch2019classification, rahbar2025survey}; in our framework, we first condition on AI signals and cast the problem as estimating a value-of-information reward \citep{raiffa1961applied, chaloner1995bayesian}.
3. \emph{LLMs for peer review}. In the context of LLM-assisted peer review, prior work focuses on improving the LLM reviewer itself \citep{poornima2023automated, jin2024agentreviewexploringpeerreview}. More recent agentic systems such as the Paper Assistant Tool \citep{jayaram2026pat} aim to automate in-depth review. These efforts are complementary to ours: we treat the LLM as a noisy signal generator, whose output (including that of such systems) enters as the AI signal $A$, and design selective human evaluation on top of it.
\begin{table}[ht]
\centering
\footnotesize
\caption{Positioning against the closest human-AI collaboration frameworks.
``Granularity'' is the unit at which the human is queried;
``AI-cond.'' indicates whether selection adapts to the AI output;
``Combines $A,H$'' indicates whether the final predictor uses both the AI
and the (partial) human signal.}
\label{tab:related_work}
\setlength{\tabcolsep}{3pt}
\begin{tabular}{@{}lcc>{\raggedright\arraybackslash}p{4.1cm}cc@{}}
\toprule
 & \textbf{Granularity} & \textbf{AI-cond.} & \textbf{Mechanism} & \textbf{Combines $A,H$} & \textbf{Budget} \\
\midrule
\citet{mozannar2020consistent}  & instance        & \checkmark & trained classifier + rejector      & no  & no       \\
\citet{wilder2020learning}      & instance        & \checkmark & joint training  & no  & no       \\
\citet{koh2020concept}          & all factors    & ---        & no learned selector        & yes & no       \\
\citet{pugnana2025deferring}    & factor subset  & \checkmark & per-concept rejector, L2D surrogate        & yes & implicit \\
\textbf{Ours}                   & \textbf{factor subset} & \checkmark & \textbf{plug-in reward / closed-form rule} & \textbf{yes} & \textbf{explicit} \\
\bottomrule
\end{tabular}
\end{table}

% =============================================================================
%  sections/problem_setting.tex
%  Target: ~0.75 pages.
%
%  Cuts vs. original §2:
%   - drop the case-construction definition of X_j(W) (Eq. construct_X) — keep
%     it for the appendix, the body just needs (A, H_pi)
%   - drop the long argument that X_pi is information-equivalent to (A, H_pi)
%
%  Notation: Z = f(A) is now introduced upfront in "Selective querying", so the
%  selection policy is pi(Z) (not pi(A)) everywhere from §2 on. Z = A recovers
%  adaptive selection; Z constant recovers non-adaptive selection.
% =============================================================================

\section{Problem Setting}
\label{sec:setting}

A decision maker observes a complex input $W$ (a manuscript, a record, a filing) and must produce an outcome $Y \in \mathbb{R}$. Following standard practice in structured human evaluation, $Y$ is determined by a vector of pre-specified, interpretable factors $X(W) = (X_1(W), \ldots, X_J(W))$ whose definitions are fixed \emph{a priori}. The map $X \mapsto Y$ is in general unknown and \emph{implicit}: different decision makers may aggregate the same factors differently --- two peer reviewers may give different overall scores from the same set of aspect-level judgments. We therefore make no functional-form assumption on $X \mapsto Y$ and only fit a predictive model from past data.

Each factor admits two kinds of evaluation:
\begin{itemize}\itemsep=1pt
\item $A(W) = (A_1, \ldots, A_J) \in \R^{J \times D_A}$: AI-generated assessments of the factors, obtained cheaply from an LLM or other automated model. These are always observed.
\item $H(W) = (H_1, \ldots, H_J) \in \R^{J \times D_H}$: human evaluations of the factors, treated as ground truth but costly to obtain.
\end{itemize}
We view $A$ as a noisy proxy for $H$. Since human evaluation is expensive, only a subset of components of $H(W)$ will be queried per instance.

As a running example, in peer review (\Cref{sec:peer_review}), $W$ is a manuscript, $Y$ is its average reviewer score, and the $J = 10$ factors are review aspects such as \emph{Claims and Evidence Rigor} or \emph{Originality and Novelty}. An LLM scores every aspect ($A$); a reviewer is asked about only a few of them ($H_S$). All notation is collected in \Cref{tab:notation} (Appendix~\ref{app:notation}).

Let $Z = f(A)$ be a fixed (possibly trivial) summary of the AI signal that the dispatching policy is allowed to consult. Two natural choices bracket the design space:
the simplest choice is $Z$ constant, in which case the policy is independent of the evaluation of AI (\emph{non-adaptive selection});
$Z = A$ itself, in which case the policy can depend arbitrarily on the realized AI signal (\emph{adaptive selection}). 
We use $S$ to denote a fixed subset of $[J]$. A \emph{selection policy} $\pi$ then maps the context $Z$ to such a subset, $\pi(Z) \subseteq [J]$, with budget constraint $|\pi(Z)| \le n_{\mathrm{sel}}$, and we write $S = \pi(Z)$ for the realized subset under $\pi$ at context $Z$. The associated human signals are $H_S := \{H_j : j \in S\}$.

For a given selection rule $\pi$, a predictor $g_\theta(A, H_{\pi})$ maps the AI signals and the queried human signals to an estimate of $Y$. In practice, $g_\theta$ is trained from past labeled data --- a corpus of historical decisions for which both $A$ and the full components of $H$ are observed --- by empirical risk minimization.
Because the input space depends on which factors are selected, the optimal predictor is itself a function of $\pi$:
\[
\theta^\ast(\pi) \;=\; \argmin_{\theta \in \Theta}\, \mathbb{E}\!\left[\ell\bigl(Y, g_{\theta}(A, H_{\pi})\bigr)\right].
\]
Given a loss $\ell$, our goal is to jointly choose the selection policy and the predictor:
\begin{equation}
\label{eq:obj}
\min_{\pi:\,|\pi(Z)| \leq n_{\mathrm{sel}}} \;
\min_{\theta \in \Theta}
\mathbb{E}\!\left[ \ell\bigl(Y, g_{\theta}(A, H_{\pi})\bigr) \right]
\;=\;
\min_{\pi:\,|\pi(Z)| \leq n_{\mathrm{sel}}}
\mathbb{E}\!\left[ \ell\bigl(Y, g_{\theta^\ast(\pi)}(A, H_{\pi})\bigr) \right].
\end{equation}
We develop both the adaptive and the non-adaptive selection rule in \Cref{sec:selection_rule}. In our experiments $g_\theta$ is an offline evaluation device: it measures how much decision-relevant information a set of signals carries, and it stands in for the human who makes the final decision in deployment (\Cref{sec:broader_impact}).

% =============================================================================
%  sections/selection_rules.tex
%  Target: ~1.5 pages.
%
%  Structure (per user direction, this revision):
%   §3.1  Reduction to contextual reward maximization UNDER SQUARED LOSS
%         - L2 loss declared upfront
%         - Narrowed Assumption 1: \exists \theta^* s.t. g_{\theta^*}=E[Y|A,H_pi]
%         - Theorem 1 stated in its squared-loss-specific form, with the
%           explicit reward R_pi(z) = E[(E[Y|A,H_pi])^2 | Z=z]
%         - Two-stage estimator referenced; full algorithm + assumptions +
%           rate theorem live in Appendix \ref{app:nonparam}
%   §3.2  Linear assumption: a closed-form reward in non-adaptive selection
%         - States explicitly that the linear-assumption analysis is for
%           non-adaptive S (a fixed subset of [J]); adaptive case in appendix
%         - Assumption 2 (linear conditional mean for fixed S)
%         - Closed form R_pi^lin \propto \beta_{pi,\perp}^T Var(H_{pi,\perp}) \beta_{pi,\perp}
%         - Interpretation: relevance × residual uncertainty
%
%  Notation: I have switched from \bar A to \widetilde A for the centered AI
%  signal. \bar A is conventionally the sample mean and is easily misread.
%  Other defensible options: A^{c}, A^{\circ}, or just (A - E[A]) inline.
% =============================================================================

\section{Optimal Selection Rules}
\label{sec:selection_rule}

Throughout this section we adopt the squared loss $\ell(y, a) = (y - a)^2$.
A standard result in regression theory says that, under squared loss, the optimal predictor of $Y$ from any feature vector is the conditional expectation;
here this means that for every fixed selection $\pi$ the solution to the inner-minimization is $\E[Y \mid A, H_\pi]$.
Combined with a mild well-specification condition on the predictor class (\Cref{assump:rich} below),
this lets us reduce the double minimization \eqref{eq:obj} to maximizing a context-dependent quantity over candidate subsets (\Cref{thm:general_reduction}).
We then specialize further to a linear outcome model and obtain a one-line decomposition that makes plain which factors dominate the optimal selection (\Cref{sec:linear_insight}).
% Full estimation algorithms and statistical rates are in Appendix~\ref{app:nonparam} (nonparametric) and Appendix~\ref{app:linear} (linear).

\subsection{Reduction to contextual reward maximization under squared loss}
\label{sec:reduction}

% Recall from \Cref{sec:setting} that $Z = f(A)$ is a fixed summary of the AI signal that the policy is allowed to consult; in our experiments $Z$ is either constant (non-adaptive selection) or equal to $A$ itself.
\begin{assumption}[Bayes predictor recoverability]
\label{assump:rich}
For every selection rule $\pi$, there exists $\theta^\star \in \Theta$ such that
\(
g_{\theta^\star}(A, H_\pi) \;=\; \E[Y \mid A, H_\pi].
\)
\end{assumption}

Assumption~\ref{assump:rich} is the standard well-specification condition: it says the predictor class $\Theta$ is rich enough that, for each fixed $\pi$, empirical risk minimization under squared loss can in principle recover the conditional-expectation regressor.

\begin{theorem}[Reduction to contextual reward maximization under squared loss]
\label{thm:general_reduction}
Under the squared loss and suppose Assumption~\ref{assump:rich} holds. Define the reward function for each subset $S: |S| \leq n_{\mathrm{sel}}$,
\begin{equation}
\label{eq:Rpi_def}
R_S(z) \;:=\; \E\!\left[\bigl(\E[Y \mid A, H_S]\bigr)^2 \,\Big|\, Z = z\right].
\end{equation}
% Denote $R_{\pi}(Z) := R_{S}(Z)$ for $\pi(Z) = S$.
Then objective \eqref{eq:obj} is equivalent to
\[
\max_{\pi:\,|\pi(Z)| \le n_{\mathrm{sel}}} \;\E\!\left[R_{\pi(Z)}(Z)\right],
\]
and the optimal rule satisfies $\pi^\ast(z) \in \arg\max_{|S| \le n_{\mathrm{sel}}} R_S(z)$ for each realization $Z = z$.
\end{theorem}

The proof is a short cascade: under Assumption~\ref{assump:rich} the inner minimization is solved by $\E[Y\mid A, H_\pi]$, the resulting Bayes risk equals $\E[\Var(Y \mid A, H_\pi)] = \Var(Y) - \Var(\E[Y\mid A,H_\pi])$ by the law of total variance, and $\Var(Y)$ is independent of selected factor $\pi$; centering and conditioning on $Z$ then yield the form \eqref{eq:Rpi_def}. The full derivation is in Appendix~\ref{app:nonparam}. The reward $R_S(z)$ is thus the second moment of the Bayes predictor given $(A, H_S)$ at context $Z = z$; solving for  an optimal selection rule reduces to estimating it over candidate subsets and contexts.

% The reward $R_S(z)$ thus has a clean meaning: it is the second moment of the Bayes predictor given AI signal $A$ and partial human signal $H_S$,
% conditional on the realized context $Z = z$. A larger $R_S(z)$ means the AI--human pair retains more decision-relevant variation about $Y$ at that context. Designing a selection rule reduces to estimating $R_S(z)$ over candidate subsets and contexts.

We estimate the reward function $R_S(z)$ via a cross-fitted two-stage procedure, which constructs an orthogonal pseudo-outcome and regresses it on the context $Z$. Full details, including the estimator, cross-fitting scheme, and statistical guarantees, are provided in Appendix~\ref{app:nonparam}. When candidate subsets are compared at the same context, we center the pseudo-outcomes across candidates before the second-stage regression; this removes a term common to all candidates, leaves the maximizer unchanged, and substantially reduces estimation variance (\Cref{rem:centering}, Appendix~\ref{app:simulation}).

% \paragraph{Two-stage estimator.}
% Given a fixed subset $S$, let $m_S(a, h_S) := \E[Y \mid A=a, H_S=h_S]$. We estimate the reward function $R_S(z)$ in two stages. The first stage produces a cross-fitted estimator $\widehat m_S$ of the conditional mean. The second stage forms an \emph{orthogonal pseudo-outcome},
% \[
% \widehat\phi_{i,S} \;:=\; 2\, Y_i\, \widehat m_{S,-i}(A_i, H_{i,S}) \;-\; \widehat m_{S,-i}(A_i, H_{i,S})^2,
% \]
% and recovers $\widehat R_S(z)$ either by regressing $\widehat\phi$ on $Z$ in the adaptive case or by averaging $\widehat\phi$ in the non-adaptive case. Here $i$ indexes a single training sample, and the subscript $-i$ denotes the \emph{cross-fitted} version of the first-stage estimator: the training sample is partitioned into folds, and $\widehat m_{S,-i}$ is fit on every fold \emph{except} the one containing~$i$, so that $\widehat m_{S,-i}$ is independent of $(Y_i, A_i, H_{i,S})$. Cross-fitting together with the orthogonal correction let first-stage error enter only at second order. The full algorithm, convergence-rate assumptions on the two regressions, and the formal rate statement are in Appendix~\ref{app:nonparam}.

\subsection{Linear assumption: a closed-form reward in non-adaptive selection}
\label{sec:linear_insight}
To understand which factors drive selection, we make two reductions, stated as formal assumptions: we restrict attention to non-adaptive selection rules (\Cref{assump:nonadaptive}) and impose a linear model on the conditional mean of $Y$ (\Cref{assump:linear}). Neither is needed for \Cref{thm:general_reduction} or for the two-stage estimator; they buy a closed form that is estimable from sample covariances.

\begin{assumption}[Non-adaptive selection]
\label{assump:nonadaptive}
The context is constant, $Z \equiv c$, so that $\pi(Z) \equiv S$ for a fixed subset $S \subseteq [J]$ with $|S| \le n_{\mathrm{sel}}$.
\end{assumption}

\begin{assumption}[Linear Model]
\label{assump:linear}
For every (fixed) subset $S \subseteq [J]$, there exist coefficients $(\gamma_S, \beta_S, c_S)$ such that
\(
\E[Y \mid A, H_S] \;=\; \gamma_S^\top A + \beta_S^\top H_S + c_S.
\)
\end{assumption}

Center the signals: $\widetilde A := A - \E[A]$ and $\widetilde H_S := H_S - \E[H_S]$. Then, orthogonalize the human signal against the AI signal,
$
H_{S,\perp} \;:=\; \widetilde H_S \;-\; \E[\widetilde H_S\, \widetilde A^\top]\, \E[\widetilde A\, \widetilde A^\top]^{-1}\, \widetilde A,
$
so that $\E[H_{S,\perp}\,\widetilde A^\top] = 0$. Under \Cref{assump:linear} the conditional mean rewrites as
\begin{equation}
\label{eq:centered_lin}
\E[Y \mid A, H_S] - \E[Y] \;=\; \gamma^\top \widetilde A \;+\; \beta_{S,\perp}^\top H_{S,\perp},
\end{equation}
where the AI-only coefficient $\gamma$ is invariant to the choice of $S$ (a tower-property argument; see Appendix~\ref{app:linear}).\footnote{The intercept from \Cref{assump:linear} drops out of \eqref{eq:centered_lin}: taking unconditional expectations of both sides, the LHS equals $0$ by the tower property, and on the RHS $\E[\widetilde A] = 0$ and $\E[H_{S,\perp}] = 0$ by construction, so any constant term must be zero.}
Plugging this into the explained-variance form of the objective and dropping the $S$-independent contribution from $\gamma^\top \widetilde A$ yields the non-adaptive closed form
\begin{equation}
\label{eq:Rpi_linear_decomp}
\boxed{\;
R_S^{\mathrm{lin}} = \beta_{S,\perp}^\top \,\Var(H_{S,\perp})\, \beta_{S,\perp}.
\;}
\end{equation}

\paragraph{Interpretation.}
For a singleton $S = \{j\}$, \eqref{eq:Rpi_linear_decomp} reduces to $\|\beta_{j,\perp}\|^2 \cdot \Var(H_{j,\perp})$ up to a scalar, and the two factors carry direct meaning:
\begin{itemize}\itemsep=2pt
% \item $\Var(H_{j,\perp})$ measures how much variation in the human signal on factor $j$
% remains \emph{unexplained by the AI signal $A$ via a linear fit}. It is large when human
% judgments on aspect $j$ vary substantially across instances while $A$ does not track that variation
% --- for example, when human scores spread across the full range $[0,1]$ but the LLM consistently
% outputs roughly the same score regardless of the manuscript. It is small when $A$ \emph{discriminates}
% between papers in the same way the human does, even if $A$ is biased: imagine an extreme case, if the LLM systematically rates
% papers $0.2$ above the human, then $|A_j - H_j|$ is large pointwise but $A_j$ still pins down $H_j$
% exactly up to a deterministic shift, so the residual variance is zero. Human input is informative on
%  aspect $j$ in the first situation and uninformative in the second.
\item $\mathrm{Var}(H_{j,\perp})$ measures the variation in the human signal on factor $j$ that is not explained by the AI signal. It is large when human judgments vary across instances but the AI fails to track that variation (e.g., human scores span a wide range while the LLM outputs are nearly constant), and small when the AI already captures the human signal up to a deterministic shift (e.g., the LLM may be systematically biased but preserves relative differences, so the residual variance is close to zero).
\item $\|\beta_{j,\perp}\|^2$ measures the \emph{predictive relevance} of that residual to the outcome.
A factor whose residual is uninformative for $Y$ is not worth querying even if the AI cannot predict it.
\end{itemize}
% Optimal dispatching therefore prioritizes factors on which the AI signal both \emph{fails to discriminate} between instances the way the human does and where the resulting residual \emph{matters} for the final decision.  We use this estimator as the practical selection rule in our experiments.
Optimal dispatching therefore prioritizes factors where the AI signal leaves substantial unexplained variation and where that variation is predictive of the outcome. In particular, the size of the AI--human gap is not what matters: an AI signal that is systematically biased but ranks instances the way humans do leaves little residual uncertainty, whereas an AI signal that barely varies across instances leaves a large one. Equation~\eqref{eq:Rpi_linear_decomp} is plug-in estimable from sample covariances and a single ridge fit; its asymptotic-normality result is in Appendix~\ref{appdx:asymp_normality}.

% =============================================================================
%  sections/peer_review_application.tex
%  Target: ~4.5 pages.
%
%  Combines the old §4 (Practical Selection Rules), §5 (Peer Review Case Study),
%  §6 (Experiments). Per user direction:
%    1. decompose peer review to J = 10 subtasks
%    2. LLM signals: A (formalized: score + logic reason) and E (unformalized
%       reasoning) — explain why we split context-dependent E from context-free A
%    3. human signals H_{j,0}, H_{j,1}, H_{j,2}; how they are constructed; offline-proxy disclosure
%    4. selection rules actually applied
%    5. dataset (ICLR 2019, 2021), train/test split, double cross-validation
%    6. results (with conservative claims)
%
%  AUTHOR_NOTE markers below indicate places where the user must adapt content
%  (in particular the A/E split was added per user direction and is not in the
%  prior draft; the description below is scaffolding that the user should refine).
% =============================================================================

\section{Application: AI-Assisted Peer Review}
\label{sec:peer_review}

We instantiate the framework on AI-assisted peer review, where reviewer effort is the binding resource and LLMs can produce rich but unreliable preliminary assessments. The goal is to predict a paper's average reviewer score $Y$ from a combination of LLM-generated and selectively elicited human signals over $J = 10$ aspects.

\subsection{Decomposing the review task into ten aspects}
\label{sec:aspects}

We decompose reviewing into $J = 10$ disjoint aspects:
\emph{Claims and Evidence Rigor}, \emph{Originality and Novelty},
\emph{Significance and Impact}, \emph{Experimental Design and Robustness},
\emph{Clarity and Writing}, \emph{Reproducibility and Transparency},
\emph{Related Work and Positioning}, \emph{Limitations},
\emph{Ethics and Societal Impact}, and \emph{Future Work Potential}.
The taxonomy is fixed before any data analysis and shared across the LLM and the human reviewer;
full definitions, standardized prompts, the controlled vocabulary of \emph{logic reasons} are recorded in Appendix~\ref{appdx:aspect_list}.

The decomposition is a sentence-level refinement of the four broad criteria of the NeurIPS
General Reviewing Guidelines (\emph{Quality}, \emph{Clarity}, \emph{Significance}, \emph{Originality}); we also recorded the correspondence in Appendix~\ref{appdx:aspect_list}.  
The motivation is that the value of human inspection is not uniform within a broad criterion:
an LLM may already be reliable on \emph{Experimental Design and Robustness}
yet remain highly uncertain on \emph{Claims and Evidence Rigor},
while both are sub-dimensions of \emph{Quality}.
The finer decomposition lets the selection rule dispatch human attention
to the exact sub-dimension where residual AI uncertainty is highest.
% Our experiments (\Cref{sec:results}) confirm that the rules indeed concentrate human effort
% on a small subset of aspects --- a finer dispatching than the four-criterion level.

\subsection{AI signals: a formalized core $A$ and an unformalized rationale $E$}
\label{sec:llm_signals}

For each paper, we prompt an LLM (full text plus title and abstract) to produce two distinct kinds of output per aspect.
\begin{itemize}\itemsep=2pt
\item \textbf{Unformalized rationale $E_j$.} The LLM's free-text explanation --- the natural-language reasoning a reviewer would normally write. $E_j$ is information-rich but \emph{context-dependent}: each $E_j$ is grounded in the specific manuscript and cannot be compared across papers without further processing.

\item \textbf{Formalized signal $A_j$.} A score $A_{j,0} \in [0,1]$ on the aspect together with a categorical \emph{logic reason} $A_{j,1}$ chosen from the controlled vocabulary defined in Appendix~\ref{appdx:aspect_list}. The logic reason is embedded into a low-dimensional vector via a fixed projection. $A_j$ is structured, comparable across papers, and \emph{context-free} in the sense that its semantics do not depend on the manuscript text beyond the score-and-code interface.

% AUTHOR_NOTE: if E_j is used downstream (e.g., as additional context for the
% human reviewer or as text input to a learned selection rule), spell out where
% it enters here. If E is collected only as documentation/audit trail, say so.
\end{itemize}
Decoupling $A$ from $E$ is methodologically deliberate. Selection rules and prediction models in our framework operate on $A$, where features have stable cross-paper meaning and standard learning algorithms can be applied. The free-form rationale $E$ is used only as a transparency artifact and an input to the human reviewer; we do not regress on it. This separation lets us study selection-rule design without entangling it with the noisier process of distilling free text into features.

% \textbf{LLM models.}
% Subject to API-budget constraints,
% we run experiments with three language models:
% GPT-5-mini, GPT-4.1, and Gemini-2.0 Flash. 
% To establish how well each LLM performs peer review autonomously,
% without any human input,
% we additionally prompt each model directly for
% (i) a final reviewer score on the $1$--$10$ scale, and
% (ii) a binary accept/reject decision. The results are reported in section~\ref{sec:results}.
\textbf{LLM models.}
We use relatively inexpensive, widely accessible models---GPT-5-mini, GPT-4.1, and Gemini-2.0 Flash---as the source of the AI signal $A$. As explained in the introduction, this choice is not only budget-driven, but central to the deployment we target at. In addition, we prompt each model directly for (i) a final reviewer score on the 1--10 scale and (ii) a binary accept/reject decision; the results are reported in Section~\ref{sec:results}.

\subsection{Human signals and the offline-proxy caveat}
\label{sec:human_signals}

For each paper-aspect pair, we record a triple $(H_{j,0}, H_{j,1}, H_{j,2})$ that mirrors the LLM's structured output:
\begin{itemize}\itemsep=2pt
\item $H_{j,0} \in \{0,1\}$: whether the reviewer deems the aspect relevant to the final evaluation;
\item $H_{j,1} \in [0,1]$: the score the reviewer would assign on that aspect, with $H_{j,1} = 0.5$ as a neutral imputation when $H_{j,0} = 0$;
\item $H_{j,2}$: a logic reason from the same controlled vocabulary as $A_{j,1}$, projected to the same low-dimensional space.
\end{itemize}
\paragraph{In our experiments (offline proxy).}
The OpenReview corpus contains only raw reviewer free-text; it does not include per-aspect scores for the ten-aspects decomposition in our framework. We therefore construct $H$ by applying an LLM \emph{extractor} to each posted review (without access to the full paper), asking it to fill in $(H_{j,0}, H_{j,1}, H_{j,2})$ for every aspect from that review. The resulting features should be read as a structured proxy for the reviewer's factor-level reasoning: they are correlated with the original review more strongly than independent annotations would be, so our absolute MAE numbers are likely optimistic relative to what the same selection rules would achieve in deployment (\Cref{sec:results}, \Cref{sec:conclusion}).

At decision time, the AI signals $A$ are first queried for the manuscript and the selection rule $\pi(Z)$ identifies a subset of aspects; a human reviewer is then asked to write a brief free-text judgment (a few sentences) on each selected aspect $j \in \pi(Z)$. The judgment is passed through the \emph{same} LLM extractor used at training time, which produces the structured triple $(H_{j,0}, H_{j,1}, H_{j,2})$ for each $j \in \pi(Z)$. We do not elicit the triple directly, because reviewers use different internal scales for $H_{j,1}$ and the prediction model was fit on extractor outputs. The reviewer's effort is bounded by $|\pi(Z)| \le n_{\mathrm{sel}}$.

\subsection{Selection rules used}
\label{sec:selection_rules_used}

\Cref{sec:selection_rule} defines the optimal subset reward $R_S$,
but evaluating it over all $\binom{J}{n_{\mathrm{sel}}}$ subsets can be computationally expensive.
We therefore use singleton approximations where appropriate, and retain full subset evaluation only for the linear closed-form rule, which is tractable here since $J=10$. We compare eight selection rules, named identically in every figure and table (\Cref{tab:rule_names}). Rules 3--4 follow from the linear decomposition in \Cref{sec:linear_insight}; rules 1--2 are nonparametric references; rules 5--6 are ablations and intuitive baselines; rules 7--8 are non-learned baselines.
\begin{enumerate}\itemsep=2pt
\item $R_j(A)$ (nonparametric, adaptive): estimate $R_{\{j\}}(A)$ via the two-stage procedure in \Cref{sec:reduction} and Appendix~\ref{subsec:estimate_Rpi_nonparam}, with pseudo-outcomes centered across aspects (\Cref{rem:centering}), and select the top $n_{\mathrm{sel}}$ aspects at the realized $A$.
\item $R_j$ (nonparametric, non-adaptive): estimate $R_{\{j\}}$ by averaging pseudo-outcomes over the training set; details are in Appendix~\ref{subsec:estimate_Rpi_nonparam}.
\item $R_S^{\mathrm{lin}}$ (linear, subset): evaluate the closed-form reward \eqref{eq:Rpi_linear_decomp} over all subsets of size $n_{\mathrm{sel}}$ and select the maximizer.
\item $R_j^{\mathrm{lin}}$ (linear, singleton): rank aspects by $\|\beta_{j,\perp}\|^2 \Var(H_{j,\perp})$ and select the top $n_{\mathrm{sel}}$.
\item $\|\beta_{j,\perp}\|$ (importance baseline): rank aspects by predictive relevance alone, ignoring residual variance.
\item Learned agreement: predict the AI--human disagreement $|A_{j,0} - H_{j,1}|$ from $A$ and select the aspects with the highest predicted disagreement; details are in Appendix~\ref{appdx:learned_agreement}.
\item Fixed subset: always query \emph{Claims and Evidence Rigor}, then \emph{Originality and Novelty}, \emph{Significance and Impact}, and \emph{Experimental Design and Robustness}, in this order ($n_{\mathrm{sel}} \le 4$).
\item Random-$k$: query $n_{\mathrm{sel}}$ aspects drawn uniformly at random, independently for each paper.
\end{enumerate}

\subsection{Dataset and training}
\label{sec:dataset}

We use $3{,}408$ ICLR papers ($1{,}417$ from ICLR 2019 and $1{,}991$ from ICLR 2021), retrieved via the OpenReview API together with reviewer reports and final decisions. ICLR releases individual reviews and decisions for both accepted and rejected papers, which keeps the underlying paper distribution unbiased relative to acceptance status. ICLR 2020 is excluded because its reviews used a four-level rating scale (1, 3, 6, 8) that is not directly comparable with the 1--10 scale of 2019 and 2021. Because of API budget, the GPT-4.1 signal covers a subset of $1{,}995$ papers ($998$ from 2019 and $997$ from 2021); all GPT-4.1 results use this subset. Author identifiers and acceptance markers are stripped before LLM ingestion.

We split the $3{,}408$ papers into $80\%$ train and $20\%$ test over $20$ random splits.
The downstream prediction model $g_\theta$ is fit on the train set
and evaluated on the held-out test set.
To prevent leakage when applying a selection rule on the train set,
selection-rule rewards (e.g., $R_j$, $R_j^{\mathrm{lin}}$, $R_S^{\mathrm{lin}}$)
are estimated using a $k$-fold cross-fitting scheme \emph{inside} the train set:
for each training sample $i$,
the reward used at training time is the one learned from the remaining $k-1$ folds,
so no sample contributes to its own reward fit;
% We use $k=5$.
% For the linear rules $R_S^{\mathrm{lin}}$ and $R_j^{\mathrm{lin}}$,
% this single-level cross-fitting suffices,
% because the reward is a closed-form function of train-fold covariances.
% For the $A$-adaptive nonparametric rule $R_j(A)$,
% % which additionally fits a regression of pseudo-outcomes on $A$,
% we use a \emph{double} cross-fitting scheme to avoid pseudo-outcome leakage;
the full description is deferred to Appendix~\ref{appdx:double_cv}.

We evaluate $g_\theta$ across both linear and non-linear regressors:
ridge regression, lasso, gradient boosting, random forests, MLPs, and \texttt{XGBoost}.
For the headline figures we report the best linear and the best non-linear regressor,
which on this problem are ridge and gradient boosting respectively;
the full per-model results are in Appendix~\ref{appdx:full_metric_table}.
For the inner regression of pseudo-outcomes on $A$ in $R_j(A)$
we use ridge with $\alpha = 1$.
% To assess sensitivity to the train/test partition,
% the entire procedure is repeated over $20$ random splits with a $20\%$ test ratio;
% reported numbers are means with $95\%$ bootstrap confidence intervals over those splits.

Peer review does not come with a ground-truth numeric label,
so we treat the average of the reviewer scores assigned to each paper (typically $3$--$4$)
as the target $Y$.
Each paper has multiple reviews, so the human signal $H$ on a given aspect $j$
must also be aggregated across reviewers.
We consider two aggregations:
\begin{itemize}\itemsep=2pt
\item \textbf{Average.} Average $(H_{j,0}, H_{j,1}, H_{j,2})$
across the reviewers of the paper.
This is our primary aggregation,
used in the headline results (\Cref{fig:main_MAE_plots}).
\item \textbf{Random.} Use the triple $(H_{j,0}, H_{j,1}, H_{j,2})$ from a single,
randomly chosen reviewer of the paper.
This is closer to the deployment scenario in which the framework
would query a single reviewer on each paper.
\end{itemize}
We report \emph{Average} results first
and the \emph{Random} counterpart in \Cref{fig:main_MAE_FIRST} (Appendix~\ref{app:extra_figures});
the two settings differ in the strength of the human input,
discussed in \Cref{sec:results}.

\subsection{Results}
\label{sec:results}

\paragraph{AI-only is far from sufficient.}
Prompted directly for a $1$--$10$ score, GPT-5-mini reaches MAE $\approx 1.87$ --- LLMs are far from calibrated as autonomous reviewers. The behavior is starker for accept/reject: of the $2{,}246$ papers in our corpus that were rejected, GPT-5-mini recommended accepting $2{,}166$ ($\mathrm{FPR} \approx 96.4\%$) and correctly rejected only $80$.\footnote{These counts use ICLR 2019 decisions recovered from the meta-reviews; an earlier version of this paper used 2019 labels that our ingestion had recorded incorrectly (Appendix~\ref{app:accept_reject}). The conclusion is unchanged.} The autonomous LLM acts as a near-``accept all'' policy, carrying essentially no negative signal. Fitting a regressor on the structured AI signal $A$ alone drops MAE to $\approx 0.70$ for GPT-5-mini (the blue dashed line in \Cref{fig:main_MAE_plots}), while the full human benchmark is below $0.50$ (the red dashed line).

\begin{figure}[t]
    \centering
    \includegraphics[width=\linewidth]{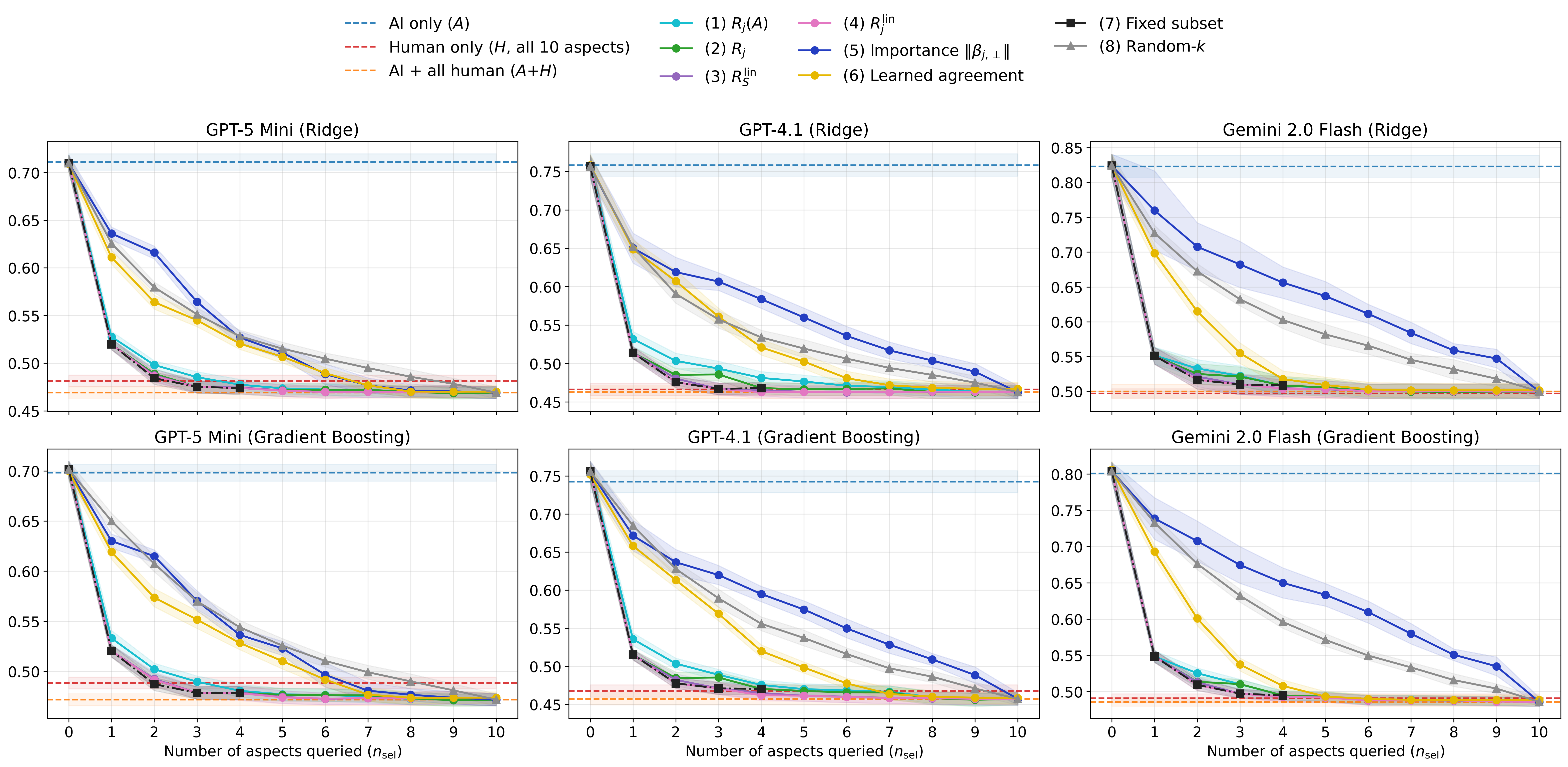}
    \caption{%
    Test MAE versus number of queried aspects under the Average aggregation, for ridge (top) and gradient boosting (bottom) across three LLMs. Rule numbers follow \Cref{sec:selection_rules_used}.
    Shaded regions are $95\%$ confidence intervals over $20$ random splits.
    Dashed lines mark the non-selection baselines and benchmark:
    AI only, human only, and combined.
    The linear rules (3)--(4) and the fixed subset (7) nearly coincide and approach the full-human benchmark with $2$--$3$ queried aspects; importance-only (5), learned agreement (6), and random-$k$ (8) need far more.}
    \label{fig:main_MAE_plots}
\end{figure}

\paragraph{Two to three queried aspects approach the full-human benchmark under optimal selection.}
Under the \emph{Average} aggregation (\Cref{fig:main_MAE_plots}), the linear rules $R_S^{\mathrm{lin}}$ and $R_j^{\mathrm{lin}}$ approach the full-human-review MAE (the red dashed line) with only $n_{\mathrm{sel}} = 2$ to $3$ queried aspects out of $J = 10$. At $n_{\mathrm{sel}}=3$, $R_j^{\mathrm{lin}}$ is below the full-human MAE for GPT-5-mini ($0.475$ vs.\ $0.481$ with ridge), within $0.004$ for GPT-4.1, and within $0.013$ for Gemini 2.0 Flash. The nonparametric rule $R_j$ is competitive but slightly worse. Centering improves the adaptive rule $R_j(A)$ over the uncentered estimator of our earlier experiments (by up to $0.011$ MAE), but remains $0.010$--$0.028$ MAE above $R_j^{\mathrm{lin}}$ at $n_{\mathrm{sel}} \in \{2,3\}$. Importance-only $\|\beta_{j,\perp}\|$ and learned agreement consistently underperform the principled rules (full results in Appendix~\ref{appdx:full_metric_table}).

\paragraph{The criterion recovers a fixed subset; plausible alternatives do not.}
On this dataset the solution of the criterion collapses to a subset shared by all papers. $R_j^{\mathrm{lin}}$ selects \emph{Claims and Evidence Rigor} first in every split for all three LLMs, and at $n_{\mathrm{sel}}=3$ it selects exactly \{Claims, Originality, Significance\} in $100\%$ of paper--split pairs. Its MAE therefore coincides with the fixed subset (7) at $n_{\mathrm{sel}}=3$ and differs from it by $-0.0004$ to $+0.0042$ at $n_{\mathrm{sel}}=2$, never significantly in favor of the learned rule (Appendix~\ref{app:fixed_baselines}). The data thus support a static prioritization of aspects; the criterion's contribution is to recover that subset from data, which plausible alternatives fail to do: random-$k$ is $0.08$--$0.17$ MAE worse at $n_{\mathrm{sel}} \in \{2,3\}$; importance alone is significantly worse than random in $11$ of $18$ configurations (3 LLMs $\times$ 2 heads $\times$ $n_{\mathrm{sel}} \le 3$); and targeting the largest predicted AI--human disagreement is $0.08$--$0.14$ MAE worse than $R_j^{\mathrm{lin}}$ at $n_{\mathrm{sel}}=2$. Conditioning the rule on paper type makes theory papers escalate Significance instead of Novelty (GPT-5-mini, Gemini) but does not reliably beat the fixed subset (Appendix~\ref{app:stratified}); a simulation shows when adaptivity does pay off (Appendix~\ref{app:simulation}).

\paragraph{Separating the human-signal text from the target.}
Because $H$ is extracted from the same reviews whose scores define $Y$, we also use a cross-reviewer holdout: $H$ comes from one review $r$ and the target is the mean score of the \emph{other} reviewers (Appendix~\ref{app:loo}). Measured as the recovered fraction of the gap between AI only and the full $A{+}H$ signal, two selected aspects recover $86.3$--$94.3\%$ of the gap under this holdout target versus $92.1$--$95.8\%$ under the standard target, and three aspects recover $88.7$--$94.5\%$ versus $93.5$--$97.8\%$ (3 LLMs $\times$ 2 heads). Under the holdout target a single reviewer's $H$ alone predicts worse than $A$ alone in five of six configurations, while combining them restores performance, which is consistent with the Random-aggregation results below.

\paragraph{Accept/reject decisions.}
The pattern also holds for the final decision (Appendix~\ref{app:accept_reject}): two selected aspects plus $A$ reach AUC $0.865$--$0.900$ and three reach $0.871$--$0.906$, versus $0.864$--$0.882$ for all ten human aspects alone and $0.665$--$0.810$ for $A$ alone; the average reviewer score itself reaches $0.955$--$0.956$.

\paragraph{Average vs.\ random reviewer.}
Under the Random aggregation (\Cref{fig:main_MAE_FIRST} in Appendix~\ref{app:extra_figures}), a single reviewer's ten aspects alone reach MAE $0.685$--$0.718$, because one reviewer is much noisier than the averaged judgment that defines $Y$. By contrast, $A$ plus two selected aspects from the same reviewer reaches $0.632$--$0.708$, below the human-only benchmark in every configuration. Calibrated AI signals can thus stabilize inconsistent individual judgments, which matters when reviewer quality is heterogeneous.

\paragraph{Robustness across years.}
Under a temporal split, the prediction model and the selection-rule rewards are fitted on the
$1{,}417$ ICLR 2019 papers and tested on the $1{,}991$ ICLR 2021 papers ($998$ and $997$ for GPT-4.1; \Cref{fig:cross_year} in Appendix~\ref{app:extra_figures}).
The qualitative ranking matches the iid case:
$R_j^{\mathrm{lin}}$ and the fixed subset dominate,
$R_j(A)$ sits between $R_j^{\mathrm{lin}}$ and the importance baseline,
and absolute MAE shifts modestly, consistent with distribution shift.
At $n_{\mathrm{sel}}=2$ the fixed subset is slightly better than $R_j^{\mathrm{lin}}$ (by at most $0.013$ MAE), because the rule fitted on 2019 alone escalates Significance instead of Novelty for GPT-5-mini and GPT-4.1.

\paragraph{Which aspects are selected.}
According to \Cref{fig:ranking} (Appendix~\ref{app:extra_figures}), the aspects with the largest reward are \emph{Claims and Evidence Rigor}, \emph{Originality and Novelty}, and \emph{Significance and Impact}. Judging whether claims are supported by correct proofs and adequate evidence, or assessing fine-grained novelty against an evolving literature, requires multi-step deliberation that inexpensive LLMs handle less reliably than surface properties such as clarity.
\emph{Ethics and Societal Impact} is never selected at the budgets we study: only $1.6$--$3.8\%$ of reviews discuss it (across our three extraction runs), so its human signal is nearly constant and a variance-based criterion cannot see it. This does not mean that ethical review has low value; \Cref{sec:broader_impact} discusses protecting such factors by design.

% =============================================================================
%  sections/broader_impact.tex  (new in the camera-ready)
%  Implements the commitments made to Reviewer R8s9 and to the ethics
%  reviewers (Hpwj, YeAz, kKUB) during the discussion period.
% =============================================================================

\section{Broader Impact and Deployment Risks}
\label{sec:broader_impact}

Our contribution is the dispatching rule. The predictor $g_\theta$ in our experiments is an evaluation device that stands in for a human decision maker we could not recruit at scale; it is not a recommendation to automate decisions. We state a design principle: in consequential domains the final decision remains with a qualified human who sees (i) the AI's structured scores and rationales for every factor, (ii) the elicited human judgments, and (iii) a record of which factors received no human evaluation, and who can override the recommendation, request more factors, or order a full review. In peer review the person queried (a reviewer) and the person who decides (the area chair) differ; dispatching only reallocates reviewer attention upstream of the area chair's decision. Using $\widehat Y$ to automate accept/reject decisions would be a misuse of the framework.

Selective evaluation directs scarce expert attention to where it is most likely to change the decision, stabilizes outcomes when individual evaluators are noisy (\Cref{sec:results}), and makes AI involvement auditable factor by factor. Risks remain even when the system works as intended. (a) Un-queried factors receive only the AI assessment, so documented LLM failure modes can pass unexamined; unlike idiosyncratic human errors, AI errors are correlated across instances, and this exposure grows as the budget shrinks. (b) A rule fitted to historical scores inherits historical reviewer and venue biases; similar average accuracy can hide new or redistributed errors across subgroups, so agreement with the historical benchmark is not evidence that no harm is introduced. (c) If the AI signal is biased with respect to language, writing style, topic, or institutional resources, or against unconventional ideas, dispatching can allocate less human scrutiny to some groups' work. (d) Showing the AI rationale can anchor human judgment.

Beyond human final authority, we recommend random full-review audits of a fraction of instances (the direct countermeasure to (a), which also regenerates training data for the reward estimator), subgroup audits of dispatch and error rates, eliciting the human judgment before showing the AI rationale, and disclosure of AI-assisted dispatching. Finally, normatively important factors such as ethics, safety, privacy, or legal compliance can be \emph{protected}: fixed in advance by the decision-making body as always human-evaluated and removed from the selection pool. A variance-based criterion cannot see such factors when their signal is rare; only $1.6$--$3.8\%$ of our reviews discuss ethics, and \emph{Ethics and Societal Impact} is never selected at the budgets we study. Protecting it costs one additional query and otherwise matches the unprotected rule to within $0.0071$ MAE across $36$ configurations (Appendix~\ref{app:protected}). We recommend this as the default, with the protected set chosen by the venue.

% \section{Discussion}
\label{sec:conclusion}

We study how to allocate limited human effort across factor-level AI outputs. Our results support a conservative interpretation: \emph{principled dispatching narrows attention to the factors that carry most of the decision-relevant content}, with the AI signal serving as a scaffold for the rest and stabilizing predictions when individual human judgments are noisy. The criterion weighs how much a factor matters for the outcome by how much of the human judgment on it the AI cannot already predict. On ICLR its solution is a fixed subset of aspects, which is a finding about this domain at this resolution (ten coarse aspects, one scalar AI score per aspect) rather than a property of the criterion.

\paragraph{Human workload.}
We make no claim about hours saved. Reading a paper is largely a fixed cost, whereas forming and writing factor-level judgments scales with the number of factors; dispatching bounds the elicited judgments to $n_{\mathrm{sel}}$ and allows expertise matching (e.g., claims to a methods expert, novelty to a domain expert). Realized time savings can only be measured in a prospective study.

\paragraph{Limitations.}
(i) Human signals are extracted from existing reviews by an LLM, and the target is the historical average score: the results are a retrospective proxy, likely optimistic, that measure how much of a full human review the selected signals recover, not whether decisions improve. A prospective study with real reviewers, which should also measure automation bias, must precede any consequential deployment. 
(ii) Experiments are limited to ICLR 2019 and 2021; other venues, and higher-stakes domains such as medicine or credit, remain untested. 
(iii) The aspect taxonomy is fixed \emph{a priori}; if important factors are omitted, no selection rule can recover them. The \emph{Ethics and Societal Impact} aspect is defined and prompted more loosely than the technical ones; its LLM score should not replace human judgment. 
(iv) The strongest-performing rule relies on linearity and non-adaptivity (\Cref{assump:nonadaptive,assump:linear}), which may be properties of this setting (Appendix~\ref{app:porting} gives a diagnostic). Adaptive rules underperform on our data; finer factors or richer AI signals could reveal regime structure where they pay off (Appendix~\ref{app:simulation}). 
(v) The framework inherits the weaknesses of its AI signal, including possible bias against unconventional ideas and susceptibility to papers written to score well on AI-assessed aspects.

\paragraph{Beyond human--AI dispatching.} The same selection criterion extends to multi-tier systems in which cheaper models provide broad but noisy signals and more expensive models --- including humans --- are invoked selectively on high-value factors. We leave this formulation to future work.

\newpage
\bibliographystyle{ACM-Reference-Format}
\bibliography{references}

\newpage
\appendix

% =============================================================================
%  appendix/nonparametric_setting_app.tex
%
%  Appendix wrapper that section-heads and \inputs the original
%  nonparametric_setting.tex. That file already contains:
%    - the eq 4 -> eq 9 cascade (i.e., the proof of Theorem 1 under squared loss)
%    - the two-stage estimator algorithm (cross-fitting + orthogonal pseudo-
%      outcome)
%    - the non-adaptive special case
%    - statistical guarantees: Assumption ass:firststage_main, ass:secondstage_main
%    - statement of Theorem thm:orth_rate_main
% =============================================================================

\section{Nonparametric Setting: Proof of Theorem~\ref{thm:general_reduction}, Two-Stage Estimator, and Statistical Rate}
\label{app:nonparam}

This appendix carries out three tasks. (i) It \emph{proves} Theorem~\ref{thm:general_reduction} via the four-step cascade that begins immediately below: under squared loss the inner minimization is solved by the conditional expectation, the law of total variance reduces the Bayes risk to an explained-variance maximization, centering and conditioning on $Z$ then yield the reward $R_S(z)$ in \eqref{eq:Rpi_def_app}. (ii) It specifies the two-stage cross-fitted estimator $\widehat R_S$ that is sketched in \Cref{sec:reduction}. (iii) It states the corresponding statistical-rate result (\Cref{thm:orth_rate_main}); the proof of that result is in {}.

We now specialize Theorem~\ref{thm:general_reduction} to the squared loss
$\ell(y,a) = (y-a)^2$.

Under the assumption~\ref{assump:rich}, for any selection rule $\pi(Z)$,
the inner minimization is solved by the conditional expectation:
\[
g_{\theta^\star}(A,H_\pi)
=
\mathbb{E}[Y \mid A,H_S],
\]
when $S = \pi(Z)$. 
The corresponding minimal risk equals
\begin{equation}
\label{eq:bayes_risk_l2}
\mathbb{E}\!\left[
(Y - \mathbb{E}[Y \mid A,H_\pi])^2
\right]
=
\mathbb{E}\!\left[
\mathrm{Var}(Y \mid A,H_\pi)
\right].
\end{equation}

By the law of total variance,
\begin{equation}
\label{eq:variance_decomp}
\mathbb{E}\!\left[
\mathrm{Var}(Y \mid A,H_\pi)
\right]
=
\mathrm{Var}(Y)
-
\mathrm{Var}\!\left(
\mathbb{E}[Y \mid A,H_\pi]
\right).
\end{equation}

Since $\mathrm{Var}(Y)$ does not depend on $\pi$,
minimizing \eqref{eq:bayes_risk_l2} is equivalent to maximizing
the explained-variance term:
\begin{equation}
\label{eq:max_explained_var}
\max_{\pi:\,|\pi(Z)| \leq n_{\mathrm{sel}}}
\;
\mathrm{Var}\!\left(
\mathbb{E}[Y \mid A,H_\pi]
\right).
\end{equation}

% \subsubsection{From min-min problem to reward function}
% \label{subsec:reward_nonparam}
% For any fixed selection policy $S$, the inner minimization problem in \eqref{eq:obj} is solved
% by the Bayes predictor
% \[
% g_{\theta^\star}(A,H_S) \;=\; \mathbb{E}\!\left[ Y \mid A, H_S \right].
% \]
% The corresponding minimal risk is given by
% \begin{equation}
% \label{eq:bayes_risk}
% \mathbb{E}\!\left[(Y-g_{\theta^\star}(A,H_S))^2\right]
% \;=\;
% \mathbb{E}\!\left[\mathrm{Var}\!\left(Y \mid A, H_S\right)\right].
% \end{equation}

% By the law of total variance,
% \begin{equation}
% \label{eq:tot_var}
% \mathbb{E}\!\left[\mathrm{Var}\!\left(Y \mid A, H_S\right)\right] = \mathrm{Var}(Y)
% -
% \mathrm{Var}\!\left(\mathbb{E}\!\left[Y \mid A, H_S\right]\right).
% \end{equation}
% Since $\mathrm{Var}(Y)$ does not depend on $S$, minimizing the Bayes risk
% \eqref{eq:bayes_risk} is equivalent to maximizing the explained-variance term:
% \begin{equation}
% \label{eq:max_explained_var}
% \max_{\pi:\,|\pi(Z)|\leq n_{\mathrm{sel}}}
% \;
% \mathrm{Var}\!\left(\mathbb{E}\!\left[Y \mid A, H_S\right]\right) 
% \end{equation}
Moreover, since $\mathbb{E}[\mathbb{E}[Y \mid A, H_\pi]] = \mathbb{E}[Y]$ and this quantity does not depend on $\pi$,
the maximization problem in \eqref{eq:max_explained_var} is equivalent to
\begin{equation}
\label{eq:centered_obj}
\begin{aligned}
\max_{\pi:\,|\pi(Z)| \leq n_{\mathrm{sel}}}
\;
\mathbb{E}\!\left[\bigl(\mathbb{E}[Y \mid A, H_S] - \mathbb{E}[Y]\bigr)^2\right]
&=
\max_{\pi:\,|\pi(Z)| \leq n_{\mathrm{sel}}}
\;
\mathbb{E}\!\left[\bigl(\mathbb{E}[Y \mid A, H_\pi]\bigr)^2\right] \\
&=
\max_{\pi:\,|\pi(Z)| \leq n_{\mathrm{sel}}}
\;
\mathbb{E}_{Z}\!\left[
\mathbb{E}\!\left[
\bigl(\mathbb{E}[Y \mid A, H_{\pi(Z)}]\bigr)^2
\,\middle|\,
Z
\right]
\right].
\end{aligned}
\end{equation}

Because the objective is an expectation over the realized selection context $Z$,
we can associate a pointwise reward with each selection rule.
Specifically, for each realization $Z = z$, define
\begin{equation}
\label{eq:Rpi_def_app}
R_{S}(z)
\;:=\;
\mathbb{E}\!\left[
\bigl(\mathbb{E}[Y \mid A, H_{S}]\bigr)^2
\;\middle|\;
Z = z \right].
\end{equation}
Maximizing \eqref{eq:centered_obj} is therefore equivalent to selecting, for each $z$,
a subset $S$ that has maximal $R_{S}(z)$. Specifically, for $n_{\mathrm{sel}} = 1$, we denote the reward for $\pi(Z) = \{j\}$ as $R_j$, i.e.,
\begin{equation}
\label{eq:Rj_def}
R_{j}(z)
\;:=\;
\mathbb{E}\!\left[
\bigl(\mathbb{E}[Y \mid A, H_{j}]\bigr)^2
\;\middle|\;
Z = z
\right].
\end{equation}

Therefore, under squared loss, the adaptive problem reduces to selecting,
for each $z$, the subset $S$ that maximizes the explained-variance reward.

% \textcolor{blue}{
% This leads to a greedy selection rule: for a given realization $Z=z$,
% select  $n_{\mathrm{sel}}$ factors with the largest values of $R_{S}(a)$.
% Formally,
% \begin{equation}
% \label{eq:greedy_rule}
% S^\star(a)
% \;\in\;
% \arg\max_{\substack{S\subseteq[J]\\ |S|=n_{\mathrm{sel}}}}
% \;
% \sum_{j\in S} R_j(a).
% \end{equation} }

\subsubsection{Estimation of reward under non-parametric setting: two-stage regression}
\label{subsec:estimate_Rpi_nonparam}
 Since we have shown that the objective in \eqref{eq:obj} breaks down to an expectation taken over the realized
selection context $Z$, we associate a reward function with any selection
rule $S$ defined pointwise in $Z$, see \eqref{eq:Rpi_def_app}. In this section, we focus on the estimation problem of these reward functions. 

\paragraph{First-stage regression and cross-fitting.}
Given a subset $S$ and define the conditional mean 
\begin{equation}
\label{eq:mpi_def}
m_{S}(a, h_{S})
\;:=\;
\mathbb{E}[Y \mid A=a,\ H_{S}=h_{S}],
\end{equation}
where $h_{S}:=\{h_j:j\in S\}$.
By definition of $R_{S}$,
\begin{equation}
\label{eq:Rpi_identity}
R_{S}(Z)
=
\mathbb{E}\!\left[
m_{S}(A, H_{S})^2
\;\middle|\;
Z
\right].
\end{equation}

To estimate $R_{S}$, we first estimate the conditional mean $m_{S}$ and then
construct an orthogonal pseudo-outcome that is robust to first-stage estimation error.
Let $\{(Y_i,A_i,H_{i,S},Z_i)\}_{i=1}^{N_R}$ denote the available training data for $R_{S}$ (where $Z_i = f(A_i)$). We use $N_R$ as the notation because this is the training data for the reward function, and we need to distinguish it from the training data for the downstream task. 

We implement a $K$-fold cross-fitting scheme.
Partition the index set $\{1,\dots,N_R\}$ into $K_R$ approximately equal-sized folds
$\mathcal{I}_1,\dots,\mathcal{I}_{K_R}$.
For each fold $k\in[K_R]$, fit a regression estimator
\[
\widehat m_{S}^{(-k)}
\quad\text{using data}\quad
\{(Y_i,A_i,H_{i,S}): i\notin \mathcal{I}_k\}.
\]
For any observation $i\in\mathcal{I}_k$, define the cross-fitted estimator
\[
\widehat m_{S,-i}
\;:=\;
\widehat m_{S}^{(-k)}.
\]
By construction, $\widehat m_{S,-i}$ is independent of $(Y_i,A_i,H_{i,S})$.

\paragraph{Orthogonal pseudo-outcome.}
Define the oracle pseudo-outcome
\[
\phi_{S}
\;:=\;
2Y\,m_{S}(A,H_{S})
-
m_{S}(A,H_{S})^2.
\]
Since $\mathbb{E}[Y\mid A,H_{S}]=m_{S}(A,H_{S})$ and through iterated expectation,
\[
\begin{aligned}
\mathbb{E}[Y\,m_{S}(A, H_{S}) \mid Z] & = \mathbb{E}\left[ \mathbb{E}[Y\,m_{S}(A, H_{S}) \mid A, H_{S} ]\mid Z \right] \\
& = \mathbb{E}\left[m_{S}(A, H_{S}) \mathbb{E}[Y \mid A, H_{S}] \mid Z \right] \\
& = \mathbb{E}\!\left[
m_{S}(A, H_{S})^2
\;\middle|\;
Z
\right],
\end{aligned}
\]
we have
\[
\mathbb{E}[\phi_{S}\mid Z]
=
\mathbb{E}\!\left[
m_{S}(A,H_{S})^2
\;\middle|\;
Z
\right]
=
R_{S}(Z).
\]

Using the cross-fitted first-stage estimator, we form the estimated pseudo-outcome
\begin{equation}
\label{eq:phi_pi_orth}
\widehat\phi_{i,S}
\;:=\;
2Y_i\,\widehat m_{S,-i}(A_i,H_{i,S})
-
\widehat m_{S,-i}(A_i,H_{i,S})^2,
\end{equation}
Later we will show that such an estimator $\widehat\phi_{i,S}$ satisfies $\mathbb{E}[\widehat\phi_{i,S}\mid Z_i]=R_{S}(Z_i)$ with up to
second-order first-stage error.

\paragraph{Second-stage regression for learning $R_{S}(\cdot)$.}
For a fixed selection rule $S$, we estimate $R_{S}(z)$ by regressing
$\widehat\phi_{i,S}$ on $Z_i$ using the sample
$\{(Z_i,\widehat\phi_{i,S})\}_{i=1}^N$.
Denote the resulting estimator by $\widehat R_{S}(\cdot)$.
In the following section, we provide statistical guarantees for this two-stage procedure.

\begin{remark}[Centering across candidates]
\label{rem:centering}
In the adaptive case $Z = A$, the law of total variance gives
$R_S(A) = \E[Y \mid A]^2 + \Var\bigl(\E[Y \mid A, H_S] \,\big|\, A\bigr)$.
The first term does not depend on $S$, so it cannot affect $\arg\max_S R_S(A)$, but each second-stage regression must still fit it. Let $\mathcal{S}$ be the family of candidate subsets compared at a context (the singletons $\{j\}$, $j \in [J]$, in our experiments). We regress the centered pseudo-outcome
\[
\widetilde\phi_{i,S} \;:=\; \widehat\phi_{i,S} \;-\; \frac{1}{|\mathcal{S}|}\sum_{S' \in \mathcal{S}} \widehat\phi_{i,S'}
\]
on $Z_i$ instead of $\widehat\phi_{i,S}$. Its conditional mean is $R_S(Z) - |\mathcal{S}|^{-1}\sum_{S'} R_{S'}(Z)$, which differs from $R_S(Z)$ by a quantity common to all candidates, so the maximizer is unchanged. In the simulation of Appendix~\ref{app:simulation}, this common term varies across contexts far more than the between-candidate contrasts that determine the maximizer, and centering raises the per-regime correct-selection rate of the adaptive rule from $0.45$--$0.49$ to $0.99$ at $n = 32{,}000$ ($n_{\mathrm{sel}}=1$). On the peer-review data it improves the adaptive rule without changing the ordering of rules (Appendix~\ref{app:centered}). All $R_j(A)$ results in the main text use the centered estimator.
\end{remark}

\paragraph{A special case: non-adaptive reward.}
A particularly simple regime arises when the abstraction function is constant,
$f(\cdot)\equiv C$, so that the selection rule $S$ is independent of the AI
signal $A$ (and hence independent of the context $Z$).
In this case, the reward function $R_S(Z)$ degenerates to a scalar,
\[
R_S
\;:=\;
\mathbb{E}\!\left[m_S(A,H_S)^2\right]
=
\mathbb{E}[\phi_S],
\]
and no second-stage regression is required.

Accordingly, we estimate $R_S$ by the empirical mean of the cross-fitted
orthogonal pseudo-outcomes,
\begin{equation}
\label{eq:Rpi_nonparam_nonadaptive}
\widehat R_S
\;:=\;
\frac{1}{N}\sum_{i=1}^N \widehat\phi_{i,S},
\qquad
\widehat\phi_{i,S}
=
2Y_i\,\widehat m_{S,-i}(A_{i}, H_{i,S})
-
\widehat m_{S,-i}(A_{i}, H_{i,S})^2.
\end{equation}

\subsubsection{Statistical Guarantee}
 We now analyze the statistical properties of the two-stage estimator given in the previous section.
The following assumptions characterize the convergence rates
of the first- and second-stage regressions.
\begin{assumption}[First-stage regression rate]
\label{ass:firststage_main}
Let $n$ denote the sample size for the training of $\widehat m_S$.
The first-stage estimator $\widehat m_S$ satisfies
\[
\|\widehat m_S - m_S\|_{4,X_{S}} = O_{\mathbb{P}}(r_{m,n})
\footnote{Throughout this subsection, for a sequence of random variables $\{X_n\}$ and a
deterministic sequence $\{a_n\}$ with $a_n>0$, we write
\[
X_n = O_{\mathbb{P}}(a_n)
\]
if for every $\varepsilon>0$ there exist constants $M_\varepsilon<\infty$ and
$n_\varepsilon$ such that
\[
\mathbb{P}\bigl(|X_n| > M_\varepsilon a_n\bigr) < \varepsilon
\quad\text{for all } n \ge n_\varepsilon.
\]
Equivalently, $X_n/a_n$ is tight.},
\]

for some deterministic sequence $r_{m,n}\to 0$, where $X_{S}:=(A,H_S)$ and
$m_S(X):=\mathbb{E}[Y\mid X_{S}]$.
\end{assumption}

\begin{assumption}[Second-stage oracle regression rate]
\label{ass:secondstage_main}
Consider the second-stage regression procedure that maps i.i.d.\ data
$\{(Z_i,U_i)\}_{i=1}^{N}$ to an estimator of $\mu(z):=\mathbb{E}[U\mid Z=z]$.
Assume that for any such distribution it produces $\widehat\mu$ satisfying
\[
\|\widehat\mu-\mu\|_{2,Z} = O_{\mathbb{P}}(r_{g,N}),
\]
for some deterministic sequence $r_{g,N}\to 0$.
\end{assumption}

\paragraph{Remark.}
Assumptions~\ref{ass:firststage_main} and~\ref{ass:secondstage_main} abstract the
statistical difficulty of the first- and second-stage regression problems through
the convergence rates $r_{m,n}$ and $r_{g,N}$.
The first-stage regression requires an $L^4$ convergence rate, which is stronger
than $L^2$ convergence and is standard in orthogonal and cross-fitted estimation,
as it ensures control of higher-order moments of the estimation error.
The second-stage oracle regression only requires $L^2$ convergence.

Both rates are compatible with a wide range of classical estimators.
In parametric settings, the rates typically scale as $n^{-1/2}$ or $N^{-1/2}$.
In nonparametric settings, the rates depend on the smoothness of the regression
function and the dimension of the covariates, reflecting the curse of dimensionality.
In particular, when $\mu(z)$ is an $s$-smooth function of $Z\in\mathbb{R}^{d_Z}$,
the optimal $L^2$ rate scales as $N^{-s/(2s+d_Z)}$, while dimension-free rates arise
only under parametric or low-dimensional structural assumptions.
Appendix~\ref{app:rates} provides representative examples.

\begin{theorem}[Orthogonal two-stage rate]
\label{thm:orth_rate_main}
Given a subset $S$. Assume $\mathbb{E}[Y^2]<\infty$ and
$\mathbb{E}[m_S(X_{S})^4]<\infty$, where $X_{S}:=(A,H_S)$ and $m_S(X_{S})=\mathbb{E}[Y\mid X_{S}]$.

Let $\widehat R_S$ be the estimator obtained by:
(i) partitioning the sample into $K$ folds,
(ii) fitting $\widehat m^{(-k)}_S$ outside fold $k$ and forming orthogonal pseudo-outcomes
\eqref{eq:phi_pi_orth}, and
(iii) estimating $R_S$ either by second-stage regression or by averaging,
depending on whether $R_S$ depends on $Z$.

Under Assumptions~\ref{ass:firststage_main}--\ref{ass:secondstage_main}, the
following hold:
\begin{enumerate}
\item[(i)]
If $R_S(z)=\mathbb{E}[(\mathbb{E}[Y\mid A,H_S])^2\mid Z=z]$ depends on $Z$, then
\begin{equation}
\label{eq:thm1_orth_rate}
\|\widehat R_S - R_S\|_{2,Z}
=
O_{\mathbb{P}}\!\bigl(r_{g,N} + r_{m,n}^2\bigr).
\end{equation}

\item[(ii)]
If $R_S$ is non-adaptive (i.e., independent of $Z$), then
\begin{equation}
\label{eq:thm1_aver_rate}
|\widehat R_S - R_S|
=
O_{\mathbb{P}}\!\left(N^{-1/2} + r_{m,n}^2\right).
\end{equation}
\end{enumerate}
\end{theorem}

\paragraph{Implication.}
Theorem~\ref{thm:orth_rate_main} shows that orthogonalization effectively
controls the impact of first-stage estimation error, which enters only at
second order. However, the overall performance of $\widehat R_S$ still
critically depends on the accuracy of the second-stage regression in $Z$.
When $R_S(Z)$ varies with $Z$ and the available sample size is limited,
nonparametric regression of $R_S$ can incur non-negligible estimation error,
which may dominate the overall rate.

Moreover, the nonparametric formulation obscures the structural content of the
reward function $R_S$ and provides limited insight into how different
selection rules $S$ compare or how optimal rules should be constructed.
To better understand the form of $R_S$ and to obtain analytically tractable
selection rules, we next consider a parametric setting under a linear outcome
model. In this regime, $R_S$ admits an explicit decomposition in terms of
variance and covariance components, which directly motivates the optimal
selection rule developed in the next section.
% Under mild moment conditions and assuming the first-stage estimator
% $\widehat m_{S}$ converges in $L^4$ at rate $r_{m,n}$, the orthogonal two-stage estimator
% $\widehat R_{S}$ satisfies
% \[
% \norm{\widehat R_{S} - R_{S}}_{2,Z}
% =
% O_{\Pp}\!\bigl(r_{g,n} + r_{m,n}^2\bigr),
% \]
% where $r_{g,n}$ is the oracle rate of the second-stage regression when the true
% pseudo-outcomes are observed.
% Importantly, this rate does not depend on the size of $S$ beyond its effect on
% the complexity of the first-stage regression.
% A formal statement and proof are given in Appendix~\ref{app:orthogonal}.

\paragraph{Proof of Theorem~\ref{thm:orth_rate_main}.}
Define the regression noise
\[
\varepsilon := Y - m_S(X_{S}),
\qquad
\mathbb{E}[\varepsilon \mid X_{S}]=0,
\]
and the first-stage estimation error
\[
\delta(X_{S}) := \widehat m_S(X_{S}) - m_S(X_{S}).
\]
On the holdout sample used for the second stage, expand the pseudo-outcome:
\begin{align*}
\widehat\phi_{i,S}
&=
2Y_i(m_S(X_{i, S})+\delta(X_{i, S})) - (m_S(X_{i, S})+\delta(X_{i, S}))^2 \\
&=
\phi_{S}(X_{i, S}) + 2\delta(X_{i, S})\varepsilon_i - \delta(X_{i, S})^2.
\end{align*}
Because of cross-fitting, $\delta(X_{i, S})$ is independent of $(Y_i,A_i,H_{i,S})$.
Taking conditional expectation given $(X_{S},\widehat m_{S})$ (we hereby omit the index $i$ for simplicity) yields
\[
\mathbb{E}[\widehat\phi_{S} \mid X_{S},\widehat m_S]
=
m_S(X_{S})^2 - \delta(X_{S})^2.
\]
Conditioning further on $(Z,\widehat m_{S})$ gives
\[
\mathbb{E}[\widehat\phi_{S} \mid Z,\widehat m_{S}]
=
R_S(Z) - b_{\widehat m_{S}}(Z),
\qquad
b_{\widehat m_{S}}(z)
:=
\mathbb{E}[\delta(X_{S})^2 \mid Z=z,\widehat m_{S}].
\]

Let $\mu_{\widehat m_S}(z):=\mathbb{E}[\widehat\phi_{S}\mid Z=z,\widehat m_S]$.
Then
\[
\widehat R_S - R_S
=
(\widehat R_S - \mu_{\widehat m_S}) - b_{\widehat m_S}.
\]
By the assumed oracle property of the second-stage regression,
\[
\|\widehat R_S - \mu_{\widehat m_S}\|_{2,Z}
=
O_{\mathbb{P}}(r_{g,N}).
\]
Moreover,
by Jensen's inequality (conditional form),
\[
b_{\widehat m_{S}}(Z)^2
=
\big(\mathbb{E}[\delta(X_{S})^2\mid Z,\widehat m_{S}]\big)^2
\le
\mathbb{E}[\delta(X_{S})^4\mid Z,\widehat m_{S}].
\]
Taking the expectation over $(Z,\widehat m_{S})$ yields
\[
\norm{b_{\hat m}}_{2,Z}^2
=
\mathbb{E}[b_{\hat m}(Z)^2]
\le
\mathbb{E}[\delta(X_{S})^4]
=
\norm{\delta}_{4,X_{S}}^4.
\]
Therefore,
\[
\|b_{\widehat m_S}\|_{2,Z}
\le
\|\delta\|_{4,X_{S}}^2
=
O_{\mathbb{P}}(r_{m,n}^2).
\]
Combining the two bounds and applying the triangular inequality, 
\[
\|\widehat R_S - R_S\|_{2, Z} \leq \|b_{\widehat m_S}\|_{2,Z} + \|\widehat R_S - \mu_{\widehat m_S}\|_{2,Z}
\]
proves \eqref{eq:thm1_orth_rate}.

In the special case where $R_S$ is constant in $Z$, the second-stage regression
step is replaced by empirical averaging. 
Conditioning on the first-stage estimates produced by cross-fitting, the
pseudo-outcomes $\{\widehat\phi_{i,S}\}_{i=1}^N$ are independent with finite
second moment; therefore, by the central limit theorem,
\[
\frac{1}{N}\sum_{i=1}^N \widehat\phi_{i,S}
-
\mathbb{E}[\widehat\phi_{S}]
=
O_{\mathbb{P}}(N^{-1/2}),
\]
 while the bias induced by the first-stage estimation remains of order
$r_{m,n}^2$ by the same orthogonality argument. Therefore, we conclude with \eqref{eq:thm1_aver_rate}.
\hfill $\square$

% =============================================================================
%  appendix/linear_model_app.tex
%
%  Appendix wrapper that section-heads and \inputs the original linear_model.tex.
%  That file contains the closed-form derivation of the linear-model reward,
%  the orthogonalization that gives the relevance-times-residual-uncertainty
%  decomposition, and the discussion of plug-in estimation. Asymptotic
%  normality of \hat R_S^{lin} is in appendix/asymp_normality_linear.tex.
% =============================================================================

\section{Linear Setting: Closed-Form Reward, Orthogonalization, and the Adaptive Case}
\label{app:linear}

This appendix derives the closed-form expression for $R_S^{\mathrm{lin}}$ used in \Cref{sec:linear_insight},
including the tower-property argument that the AI-only coefficient $\gamma$ is invariant to $S$,
the orthogonalization step that yields \eqref{eq:Rpi_linear_decomp},
and a brief discussion of the adaptive case.
We then collect the asymptotic-normality result for the plug-in estimator
$\widehat R_S^{\mathrm{lin}}$ (formerly Appendix~C),
and a brief table of regression-rate examples cited from the rate theorem (formerly Appendix~D).

\subsection{Closed-form derivation, orthogonalization, and the adaptive case}
\label{app:linear:closed_form}

In this section, we analyze the reward function under a linear model. We first derive an explicit expression, and then provide an interpretation.

\subsubsection{Reward function under linear assumption}

We assume the conditional mean is linear in the selected features: for any subset $S \subseteq \{1,\ldots,J\}$, there exist parameters $\gamma_S, \beta_S, c_S$ such that
\begin{equation}
\label{eq:linear_assumption}
   \E[Y|A, H_{S}] = \gamma_{S}^T A + \beta_{S}^T H_{S} + c_{S} := \theta_{S}^T \hat{X}_{S}, 
\end{equation}
where $\theta_{S} = (\gamma_{S}^T,  \beta_{S}^T, c_{S})^T, \hat X_{S} = (A^T, H_{S}^T, 1)^T$. 
By the tower property,
\begin{align*}
    \E[\hat{X}_{S} Y] & = \E\bigl[ \E[\hat{X}_{S}Y |\hat{X}_{S} ] \bigl] \\
    & = \E\bigl[ \hat{X}_{S}\E[Y |\hat{X}_{S} ] \bigl] = \E[\hat{X}_{S} \hat{X}_{S}^T]\theta_{S},
\end{align*}
If $\E[\hat{X}_{S} \hat{X}_{S}^T]$ is invertible, then
\begin{equation}
\label{eq:theta_pi_fullrank}
    \theta_{S} = \E[\hat{X}_{S} \hat{X}_{S}^T]^{-1} \E[\hat{X}_{S} Y],
\end{equation}
otherwise, one may use ridge regularization with parameter $\lambda>0$ to obtain
\begin{equation}
\label{eq:theta_pi_estimate}
    \hat{\theta}_{S} = \left(\E[\hat{X}_{S} \hat{X}_{S}^T] + \lambda I \right)^{-1} \E[\hat{X}_{S} Y].
\end{equation}
In the theoretical part, we proceed with \eqref{eq:theta_pi_fullrank}. 
Assuming the model class is rich enough to recover $g_{\theta^*}(A,H_S)=\E[Y\mid A,H_S]$, 
we follow Section~\ref{sec:reduction} (and Appendix~\ref{app:nonparam}) to obtain
\begin{align}
R_{S}^{\text{(lin)}}(z)
\;& =\; 
\mathbb{E}\!\left[
\bigl(\mathbb{E}[Y \mid A, H_{S}]\bigr)^2
\;\middle|\;
Z = z \right] \nonumber \\
& = \; \theta_{S}^T\; \mathbb{E}\!\left[
\hat{X}_{S} \hat{X}_{S}^T
\;\middle|\;
Z = z \right] \theta_{S} \label{eq:Rpi_linear}
\end{align}
If we consider the non-adaptive selection rule, i.e., $S$ does not vary with $Z$,
then the reward is reduced to 
\begin{align}
    R_{S}^\text{(lin)} & = \; \theta_{S}^T\; \mathbb{E}\![
\hat{X}_{S} \hat{X}_{S}^T ]\; \theta_{S} \label{eq:R_linear_1} \\
& = \E[\hat{X}_{S} Y]^T \E[\hat{X}_{S} \hat{X}_{S}^T]^{-1} \E[\hat{X}_{S} Y] \label{eq:R_linear_2}
\end{align}

Notably, when solving for optimal non-adaptive selection rules, the same reward arises even without assuming linearity of the data-generating process, if the predictor class is restricted to linear models. 
In this case, \eqref{eq:obj} reduces to
\begin{equation}
\label{eq:obj_linear}
\min_{S:\,|S| \leq n_{\mathrm{sel}}} \;
\min_{\theta \in \Theta} 
\mathbb{E}\!\left[ \left( Y - \hat X_S^\top \theta \right)^2\right],
\end{equation}
where $\hat X_{S} = (A^\top, H_{S}^\top, 1)^\top$. The inner minimizer is given by the normal equations,
\begin{equation*}
    \theta^* = \E[\hat X_{S} \hat X_{S}^T]^{-1} \E[\hat X_{S} Y].
\end{equation*}
Substituting back, the objective reduces to
\begin{equation*}
\max_{S:\,|S| \leq n_{\mathrm{sel}}} \;
\E[\hat X_{S} Y]^T \E[\hat X_{S} \hat X_{S}^T]^{-1} \E[\hat X_{S} Y],
\end{equation*}
which coincides with the reward in \eqref{eq:R_linear_2}.

\subsubsection{Insight from the non-adaptive reward}
\label{subsubsec:reward_linear}
The linear model \eqref{eq:linear_assumption} can be rewritten in centered form as
\begin{equation*}
    \E[Y|A, H_{S}] - \E[Y] = \gamma_{S}^T (A - \E[A]) + \beta_{S}^T (H_{S} - \E[H_{S}])
\end{equation*}

If we orthogonalize $\bar{H}_{S}:= H_{S} - \E[H_{S}]$ with respect to $\bar{A} := A - \E[A]$, 
\begin{equation*}
    H_{S,\perp}
:= \bar{H}_S
- \E[\bar{H}_S \bar{A}^\top]\;
\E[\bar{A}\bar{A}^\top]^{-1}\;
\bar{A},
\end{equation*}
so that 
\begin{equation*}
    \E[H_{S, \perp} \bar{A}^T] = 0 \; \text{and} \; \E[H_{S, \perp}] = 0.
\end{equation*}
Then there exist $\gamma$ (independent of $S$) and $\beta_{S,\perp}$ such that
\[
\E[Y|A, H_{S}] - \E[Y] = \gamma^T (A - \E[A]) + \beta_{S, \perp}^T H_{S, \perp}.
\]
Because we already know that there exists $\gamma_{S, \perp}, \beta_{S, \perp}$ such that 
\[
\E[Y|A, H_{S}] - \E[Y] = \gamma_{S, \perp}^T \bar{A} + \beta_{S, \perp}^T H_{S, \perp},
\]
then through tower property, 
\begin{align*}
\E[Y \bar{A}^T] - \E[Y] \E[\bar{A}^T] & = \E[ \E[Y|A, H_{S}] \bar{A}^T] - \E[Y] \E[\bar{A}^T] \\
& = \gamma^T \E[\bar{A} \bar{A}^T] + \beta^T_{S, \perp} \E[H_{S, \perp} \bar{A}^T] = \gamma^T \E[\bar{A} \bar{A}^T].
\end{align*}
Therefore, the coefficient before $\bar{A}$ is consistent across $S$ and is given by 
\[
\gamma = \E[\bar{A}\bar{A}^T]^{-1} \left(\E[Y\bar{A}] - \E[Y] \E[\bar{A}]\right).
\]

For $S$ that does not vary with $A$, recall \eqref{eq:max_explained_var}, we have 
\begin{equation*}
\begin{aligned}
\pi^* & = \argmax_{S: |S| \leq n_{\mathrm{sel}}} \Var\!\left(\mathbb{E}\!\left[Y \mid A, H_S\right]\right) \\
& = \argmax_{S: |S| \leq n_{\mathrm{sel}}} \beta_{S, \perp}^T \Var(H_{S, \perp} )\beta_{S, \perp}
\end{aligned} 
\end{equation*}

For each singleton selection $S=\{j\}$, the quantity $\Var(H_{j,\perp})$ captures the residual variation in the human signal after removing the component explained by $A$, i.e., the information in $H_j$ not captured by the AI signal. The coefficient $\beta_{j,\perp}$ measures the predictive relevance of this residual information for $Y$. Therefore, the optimal selection prioritizes aspects that combine high residual variation (uncertainty relative to AI) and strong predictive relevance.

% \[
% \E[Y] = \E[\theta_{S}^T \hat X_{S}] = \theta_{S}^T \E[\hat X_{S}], 
% \]
% therefore, the optimal subset $S$ selected according to \eqref{eq:R_linear_1} is also equivalent to 
% \begin{equation*}
%      \pi^* = \argmin_{S} \left(R_{S} - \E[Y]^2 \right)
%    = \argmin_{S}\; \theta_{S}^T \;\Var(\hat X_{S}) \; \theta_{S}
% \end{equation*}
% \begin{equation*}
% \begin{aligned}
%     \\
%     & = \argmin_{S} \gamma_{S}^T \Var(A) \gamma_{S} + 2 \gamma_{S}^T \
% \end{aligned}
% \end{equation*} 

% Under a query budget $n_{\mathrm{sel}}$, we aim to reduce the Bayes risk given in \eqref{eq:bayes_risk}.
% As in \eqref{eq:max_explained_var}, this is equivalent to maximizing the explained variance
% \begin{equation}
% \label{eq:equiv_max_explained}
% S^\star \in \arg\max_{|S| \leq n_{\mathrm{sel}}}\ 
% \mathrm{Var}\!\left(\mathbb{E}\!\left[Y\mid A, H_S\right]\right) \coloneq \arg\max_{|S| \leq n_{\mathrm{sel}}} R_S.
% \end{equation}

\subsubsection{Estimate of the selection criterion} \label{sec:selection-criterion-estimate}
In our experiment, we deploy the non-adaptive selection rule derived under the linear assumption. 
We use the formula given in \eqref{eq:R_linear_1}
\[
R_{S}^\text{(lin)} = \theta_{S}^T \E\bigl[ \hat X_{S} \hat X_{S}^T \bigl] \theta_{S},
\]
to estimate the reward function and select the optimal subset. In order to do so, we first estimate the coefficient $\theta_{S}$ and $\E\bigl[ \hat X_{S} \hat X_{S}^T \bigl]$ and calculate the reward function $R_{S}^\text{(lin)}$.
Still, let $\{(Y_i,A_i,H_{i,S})\}_{i=1}^{N_R}$ denote the available training data for $R_{S}^\text{(lin)}$, then 
\begin{equation}
\label{eq:Rpi_lin_estimator}
\hat R_{S}^\text{(lin)} =  \hat\theta^T_{S} \left(\frac{1}{N_R}\sum_{i = 1}^{N_R} \hat X_{S,i} \hat X_{S,i}^T \right)\hat\theta_{S},
\end{equation}
where $\hat{\theta}_{S}$ is obtained via running ridge regression on $(A_i, H_{i,S}) \to Y$, as we mentioned in \eqref{eq:theta_pi_estimate}.

% \textcolor{red}{To be modified according to the $\hat{\theta}_{S}$ and $\hat{X}_{S}$.}
% Assume we are given access to a training dataset of size $n$ with data $A^{(n)} \in \mathbb{R}^{n \times d_A}$, $H^{(n)} \in\mathbb{R}^{n \times d_H}$, and target signal $Y^{(n)} \in \mathbb{R}^n$.
% This dataset may also be used for fitting a downstream model mapping $(A, H_S) \to Y$.
% For fixed $S$, we wish to assign an estimate of $R^\text{(lin)}_{S, i}$, $i \in [n]$ to each sample of the training set, and produce a selection rule to apply to test samples.

% To ensure proper generalization of the downstream model, we use $K$-fold cross validation to fit $R^\text{(lin)}_S$, and assign out-of-fold predictions to each training sample.
% For test samples, we refit the criterion on the full training set.
% In particular,
% {\small
% \begin{equation}
%     \hat{R}^\text{(lin)}_S = \hat{\gamma}(S)^\top \widehat{\mathrm{Var}}(A) \hat{\gamma}(S) + 2 \sum_{j \in S} \hat{\gamma}(S)^\top \widehat{\mathrm{Cov}}\left(A^{(n)}, H_j^{(n)}\right) \hat{\beta}_j + \sum_{i, j \in S \times S} \hat{\beta}_i^\top\widehat{\mathrm{Cov}}\left(H^{(n)}_i, H^{(n)}_j\right) \hat{\beta}_j,
% \end{equation}}
% where $\hat{\gamma}(S)$ and $\hat{\beta}_j$ are  estimates of $\tilde{\gamma}(S)$ and $\tilde{\beta}_j$ computed via ridge regression models on the $k$\textsuperscript{th} training fold, and $\widehat{\mathrm{Cov}}$, $\widehat{\mathrm{Var}}$ are sample covariance and variance, respectively.

\subsubsection{Asymptotic normality of the estimate}

The estimate of the criterion $\hat{R}^\text{(lin)}_S$ satisfies asymptotic normality: for fixed $S$, $\sqrt{N_{R}} \left(\hat{R}^\text{(lin)}_S - R^\text{(lin)}_S\right) \stackrel{d}{\to} \mathcal{N}(0, \sigma^2)$,
where the asymptotic variance $\sigma^2$ is a function of various covariances between $A$, $H$, and $Y$.

The result is standard; the key insight is that the estimated reward $\hat{R}^\text{(lin)}_S$ is purely a function of $\hat{S}$, the joint sample covariance matrix of $(A, H, Y)$.
We first use the central limit theorem to state the asymptotic normality of $\hat{S}$.
As all components of $\tilde{\gamma}$ and $\tilde{\beta}$ are obtained via ordinary least squares, all terms of $\hat{R}^\text{(lin)}_S$ are differentiable functions of $\hat{S}$, making $\hat{R}^\text{(lin)}_S = f(\hat{S})$ for differentiable function $f$.
Then, the delta method can be applied to finish the proof.
See Appendix \ref{appdx:asymp_normality} for full details.

% --- Merged-in: asymptotic-normality of \hat R_S^{lin} (was Appendix C) -------
\providecommand{\Ebb}{\mathbb{E}}
\subsection{Proof: Asymptotic Normality of Linear Model Reward}\label{appdx:asymp_normality}

For this section, assume the linear setting described in Section~\ref{sec:linear_insight} (and worked out in detail in Appendix~\ref{app:linear}) with fixed human subset $S$, and define $M$ as the vector produced by vertically stacking $A$, $H_S$, and $Y$.
Further, define $X \coloneq (A, H_S)$.

WLOG, assume that $M$ is zero-mean, as the mean can be incorporated into the intercept.
Additionally, to avoid special handling for the constant $C$, we augment $A$ with a $1$ and thus incorporate $C$ into $\gamma$.

Define $\Sigma$ as $\Ebb[M M^\top]$, and define sample covariance 
\[\hat{S} \coloneq \frac{1}{n} \sum_{k \in [n]} M^{(n)}_k \left(M^{(n)}_k\right)^\top.\]
In this section, we will use, e.g., $\Sigma_{A, H}$ to denote the slice corresponding to the covariance between $A$ and $H$.
Define ``expander operation'' $P$ such that $P_{A, H} \mathrm{vec}(\Sigma_{A, H})$ is equal to $\mathrm{vec}(\Sigma)$, with all but the $\Sigma_{A, H}$ component zeroed out.

We now state the main asymptotic normality result:
\begin{theorem}[Asymptotic normality of $\hat{R}^\text{(lin)}_S$]
The estimated selection criterion under a linear model satisfies
\[\sqrt{n}\left(\hat{R}^\text{(lin)}_S - R^\text{(lin)}_S\right) \stackrel{d}{\to} \mathcal{N}(0, \sigma^2),\qquad \sigma^2 = \alpha^\top \mathrm{Var}(\mathrm{vec}(MM^\top)) \alpha,\]
with $\alpha$ given by
\begin{align*}
   \alpha &= 2 P_{X,Y} \hat{\theta}(S) - P_{X,X}(\hat{\theta}(S)^\top \otimes \Sigma_{X,X}^{-1})^\top \Sigma_{X,Y}  .
\end{align*}

% {\sc Proof Sketch.}
% \textup{
    
% }

\begin{proof}
    We first rewrite $\hat{R}^\text{(lin)}_S$ as
    \begin{align*}
        \hat{R}^\text{(lin)}_S &= f(\hat{S}) \coloneq \hat{\theta}(S)^\top \hat{S}_{X,X} \hat{\theta}(S) = \hat{S}_{X,Y}^\top \hat{S}_{X,X}^{-1} \hat{S}_{X,X} \hat{S}_{X,X}^{-1} \hat{S}_{X,Y} \\
        &=  \hat{S}_{X,Y}^\top \hat{S}_{X,X}^{-1} \hat{S}_{X,Y} = \hat{\theta}(S)^\top \hat{S}_{X,Y}.
    \end{align*}

    As such $f(\hat{S})$ is a continuous function of $\hat{S}$, and it can be easily verified that $f(\Sigma)= R_S^\text{(lin)}$ (the asymptotic value of $\hat{R}_S^\text{(lin)}$, so we can apply the delta method to evaluate $\sqrt{n}(\hat{R}^\text{(lin)}_S - R^\text{(lin)}_S)$.
    
    By the central limit theorem, 
    \[\sqrt{n}\,\mathrm{vec}(\hat{S} - \Sigma) \stackrel{d}{\to} \mathcal{N}(0, \Gamma),\quad \Gamma \coloneq \mathrm{Var}(\mathrm{vec}(MM^\top)).\]

    So, by the first-order delta method,
    \[\sqrt{n}(\hat{R}^\text{(lin)}_S - R^\text{(lin)}_S) \stackrel{d}{\to} \mathcal{N}(0, (\nabla f(\Sigma))^\top \Gamma (\nabla f(\Sigma))),\]
    where we take $\nabla f(\Sigma)$ to be the gradient of $f$ with respect to the vectorized version of $\Sigma$.

    Now, all that remains is to compute $\nabla f(\Sigma)$.
    % Define notation $\partial \hat{\gamma}(S)$ to represent the derivative of $\hat{\gamma}(S)$ with respect to $\mathrm{vec}(\hat{S})$, evaluated at $\Sigma$.
    
    % Writing $f(\Sigma)$ in terms of the vectorized $\Sigma$,
    % \[f(\Sigma) = (\hat{\theta}(S) \otimes  \hat{\theta}(S))^\top \Sigma_{X,Y}.\]
    Evaluating the product rule,
    \begin{align*}
        \nabla f(\Sigma) &= P_{X,Y} \hat{\theta}(S) + \left(\tfrac{\partial}{\partial \Sigma} \hat{\theta}(S)^\top \right)  \Sigma_{X,Y}.
    \end{align*}
    As $\hat{\theta}(S) = \hat{S}_{X,X}^{-1} \hat{S}_{X,Y}$,
    \begin{align*}
        \tfrac{\partial}{\partial\Sigma} \hat{\theta} &= \tfrac{\partial}{\partial\Sigma} \left(\Sigma_{X,X}^{-1}\Sigma_{X,Y}\right) = -((\Sigma_{X,Y}^\top \Sigma_{X,X}^{-\top}) \otimes \Sigma_{X,X}^{-1}) P_{X,X}^\top  + (\Sigma_{X,X}^{-1}) P_{X,Y}^\top \\
        &= -(\hat{\theta}(S)^\top \otimes \Sigma_{X,X}^{-1})  P_{X,X}^\top  + (\Sigma_{X,X}^{-1}) P_{X,Y}^\top.
    \end{align*}
    Putting it together,
    \begin{align*}
        \nabla f(\Sigma) &= P_{X,Y} \hat{\theta}(S) + \left(-P_{X,X}(\hat{\theta}(S)^\top \otimes \Sigma_{X,X}^{-1})^\top + P_{X,Y} \Sigma_{X,X}^{-\top} \right) \Sigma_{X,Y} \\
        &= 2 P_{X,Y} \hat{\theta}(S) - P_{X,X}(\hat{\theta}(S)^\top \otimes \Sigma_{X,X}^{-1})^\top \Sigma_{X,Y} 
    \end{align*}
\end{proof}
    
\end{theorem}

% --- Merged-in: examples of regression rates (was Appendix D) -----------------
\subsection{Examples of Regression Rates}
\label{app:rates}

For a random vector $X$ with distribution $P_X$ and $p \ge 1$, we define
\[
\|f\|_{p,X} := \bigl( \mathbb{E}[|f(X)|^p] \bigr)^{1/p}.
\]
Note that $\|f\|_{2,X} \le \|f\|_{4,X}$ by H\"older's inequality.
Thus, an $L^4$ convergence rate is strictly stronger than an $L^2$ rate and implies
control over higher moments of the estimation error.
Such conditions are commonly imposed in orthogonal or debiased machine learning
to ensure that second-order remainder terms are asymptotically negligible.

The convergence rates in Assumptions~\ref{ass:firststage_main}
and~\ref{ass:secondstage_main} are compatible with a wide range of classical
regression estimators.
In parametric models, $\sqrt{n}$-type rates are attainable.
In nonparametric settings, the rates depend on the smoothness of the regression
function and the dimension of the covariates, reflecting the curse of dimensionality
\cite{tsybakov2009introduction,gyorfi2002distribution}.

\begin{table}[h!]
\centering
\caption{Representative regression rates for Assumptions~\ref{ass:firststage_main}--\ref{ass:secondstage_main}}
\label{tab:rates}
\begin{tabular}{lllll}
\toprule
Target & Function class & Dimension & Norm & Rate \\
\midrule
$m_S(A,H_S)$ & Linear / GLM & fixed & $L^4$ & $n^{-1/2}$ \\
$m_S(A,H_S)$ & Sparse linear & $p$ & $L^4$ & $\sqrt{\tfrac{s\log p}{n}}$ \\
$m_S(A,H_S)$ & H\"older $C^s$ & $d_X$ & $L^4$ & $n^{-\frac{s}{2s+d_X}}$ \\
$m_S(A,H_S)$ & RKHS & eff.\ dim. & $L^4$ & $n^{-\frac{\alpha}{2\alpha+1}}$ \\
$\mu(Z)$ & Linear / GLM & $d_Z$ & $L^2$ & $N^{-1/2}$ \\
$\mu(Z)$ & Sparse linear & $d_Z$ & $L^2$ & $\sqrt{\tfrac{s\log d_Z}{N}}$ \\
$\mu(Z)$ & H\"older $C^s$ & $d_Z$ & $L^2$ & $N^{-\frac{s}{2s+d_Z}}$ \\
$\mu(Z)$ & Additive model & $d_Z$ & $L^2$ & $N^{-\frac{s}{2s+1}}$ \\
$\mu(Z)$ & Single-index & $d_Z$ & $L^2$ & $N^{-\frac{s}{2s+1}}$ \\
\bottomrule
\end{tabular}
\end{table}

In the nonparametric regression model with $s$-smooth regression functions over
$\mathbb{R}^{d_Z}$, the minimax risk under squared $L^2$ loss satisfies
\[
\inf_{\widehat\mu} \sup_{\mu \in \mathcal H^s}
\mathbb E \|\widehat\mu - \mu\|_{2,Z}^2
\asymp N^{-2s/(2s+d_Z)},
\]
which implies the $L^2$ norm convergence rate
$\|\widehat\mu - \mu\|_{2,Z} = O_{\mathbb P}(N^{-s/(2s+d_Z)})$
\citep{stone1982optimal,tsybakov2009introduction}.

Dimension-free $N^{-1/2}$ rates arise only under parametric assumptions or when
the regression function depends on the covariates through a sufficiently
low-dimensional structure, such as additive or single-index models.

Table~\ref{tab:rates} summarizes typical convergence rates for the sequences
$r_{m,n}$ and $r_{g,N}$ under standard assumptions from statistical learning theory.

\newcolumntype{Y}{>{\RaggedRight\arraybackslash}X}
\newcommand{\stackcell}[1]{%
  \begin{tabular}[t]{@{}l@{}}#1\end{tabular}%
}

\newpage
\section{Implementation Details}

\paragraph{Data access, licenses, and compute.}
Papers, reviews, and decisions for ICLR 2019 and 2021 were retrieved from OpenReview through its public API, under the OpenReview terms of use; we do not redistribute them. GPT-5-mini, GPT-4.1, and Gemini-2.0 Flash were accessed through their providers' commercial APIs under the respective terms of service; there was no local model inference. All selection-rule estimation and prediction experiments run on CPUs, without GPUs. As a reference point, rerunning one selection rule over the 20 random splits and all budgets $n_{\mathrm{sel}} \le 10$ takes roughly 20 CPU-seconds per LLM with the ridge head and roughly 6 CPU-minutes with the gradient-boosting head (hyperparameter cross-validation included), on a single core of a standard x86 CPU.

\subsection{Aspects: Definitions and Relation to NeurIPS Reviewer Criteria}\label{appdx:aspect_list}

\paragraph{Relationship to NeurIPS reviewer criteria.}
The NeurIPS 2026 General Reviewing Guidelines ask reviewers to assess submissions along four broad
criteria: \emph{Quality}, \emph{Clarity}, \emph{Significance}, and \emph{Originality}.
We anchor the correspondence below specifically to NeurIPS 2026;
other major venues (e.g., ICML, ICLR) organize their reviewer prompts somewhat differently,
so we do not claim the same alignment carries over verbatim.
Our ten aspects are a systematic refinement of the four NeurIPS criteria, designed to be
(i) \emph{comprehensive} --- jointly covering the sub-questions the four NeurIPS criteria pose
to the reviewer, with our \emph{Ethics and Societal Impact} aspect aligned with the NeurIPS
Broader Impact solicitation that is gathered separately from the four general criteria --- and
(ii) \emph{mutually disjoint} --- each aspect targets a distinct facet of the manuscript.
Table~\ref{tb:aspect_mapping} shows the high-level mapping;
Figure~\ref{fig:aspect_overlay} drills down to a sentence-level overlay
onto the actual NeurIPS 2026 prompt wording.

\begin{table}[H]
\centering
\small
\begin{tabularx}{\linewidth}{@{}lY@{}}
\toprule
\textbf{NeurIPS 2026 Criterion} & \textbf{Our Aspects (sub-dimensions)} \\
\midrule
Quality &
  Claims and Evidence Rigor;\quad
  Experimental Design and Robustness;\quad
  Reproducibility and Transparency;\quad
  Limitations \\
\addlinespace
Clarity &
  Clarity and Writing \\
\addlinespace
Significance &
  Significance and Impact;\quad
  Future Work Potential \\
\addlinespace
Originality &
  Originality and Novelty;\quad
  Related Work and Positioning \\
\addlinespace
Broader Impact (separate solicitation) &
  Ethics and Societal Impact \\
\bottomrule
\end{tabularx}
\caption{%
High-level mapping between the four NeurIPS 2026 reviewer criteria
(plus the separately solicited Broader Impact statement) and our ten aspects.
Each aspect is a strict sub-dimension of exactly one criterion,
ensuring comprehensiveness and disjointness.
Finer granularity lets the selection rule dispatch human effort
to exactly the sub-dimension where the AI signal is least reliable,
rather than forcing an all-or-nothing query at the coarser criterion level.}
\label{tb:aspect_mapping}
\end{table}

\paragraph{Sentence-level overlay onto the NeurIPS 2026 prompts.}
Figure~\ref{fig:aspect_overlay} reproduces the NeurIPS 2026 General Reviewing Guidelines verbatim
and color-codes each individual sub-question by the aspect it activates.
Every NeurIPS sub-question maps onto exactly one of our aspects,
and nine of the ten aspects are activated by at least one sub-question
(the tenth, \emph{Ethics and Societal Impact},
is solicited separately through the Broader Impact statement).
This sentence-by-sentence overlay justifies the claim that the ten-aspect taxonomy
is a strict refinement of the four NeurIPS criteria
rather than an arbitrary alternative classification.

The decomposition serves a concrete purpose within our framework:
a selection rule operating on four broad criteria must either elicit human judgment
on an entire criterion or skip it;
one operating on ten aspects can route human attention to the specific sub-dimension
where the AI signal leaves the greatest residual uncertainty.
As an example, \emph{Clarity and Writing} is often reliably assessed by an LLM,
while \emph{Claims and Evidence Rigor} --- also a sub-dimension of Quality ---
is precisely where LLMs are documented to be over-optimistic.
Our experimental results (\Cref{sec:results}) confirm that the selection rules consistently identify
\emph{Claims and Evidence Rigor} and \emph{Originality and Novelty} as the highest-value aspects to query,
a distinction that would be invisible at the coarser four-criterion level.

\begin{figure}[H]
\centering
\begin{tcolorbox}[
  enhanced jigsaw, breakable,
  colback=white, colframe=black!55,
  title={NeurIPS 2026 General Reviewing Guidelines, color-coded to our ten aspects},
  fonttitle=\bfseries\small,
  boxrule=0.5pt, arc=2pt,
  left=6pt, right=6pt, top=4pt, bottom=4pt
]
\small
\textbf{Quality.}\quad
\hlaspect{aspRigor}{%
  Is the submission technically sound?
  Are claims well supported
  (e.g., by theoretical analysis or experimental results)?}
\hlaspect{aspEDR}{Are the methods used appropriate?}
\hlaspect{aspLim}{%
  Is this a complete piece of work or work in progress?
  Are the authors careful and honest about evaluating
  both the strengths and weaknesses of their work?}

\medskip
\textbf{Clarity.}\quad
\hlaspect{aspClarity}{%
  Is the submission clearly written?
  Is it well organized?
  Does it adequately inform the reader?}
\hlaspect{aspRepr}{%
  A superbly written paper provides enough information
  for an expert reader to reproduce its results.}

\medskip
\textbf{Significance.}
\hlaspect{aspImpact}{Are the results impactful for the community?}
\hlaspect{aspFWP}{%
  Are others (researchers or practitioners) likely
  to use the ideas or build on them?}
\hlaspect{aspImpact}{%
  Does the submission address a difficult task in a better way than previous work?
  Does it advance our understanding/knowledge on the topic in a demonstrable way?
  Does it provide unique data, unique conclusions about existing data,
  or a unique theoretical or experimental approach?}

\medskip
\textbf{Originality.}\quad
\hlaspect{aspNovelty}{%
  Does the work provide new insights, deepen understanding,
  or highlight important properties of existing methods?}
\hlaspect{aspRW}{%
  Is it clear how this work differs from previous contributions,
  with relevant citations provided?}
\hlaspect{aspNovelty}{%
  Does the work introduce novel tasks, problem framings, metrics,
  or methods that advance the field?
  Does this work offer a novel combination of existing techniques,
  and is the reasoning behind this combination well-articulated?}

\medskip
\hrule
\smallskip
\noindent\textbf{Legend (aspect $\to$ color).}\quad
\hlaspect{aspRigor}{Claims and Evidence Rigor}\,$\bullet$\,
\hlaspect{aspEDR}{Experimental Design and Robustness}\,$\bullet$\,
\hlaspect{aspRepr}{Reproducibility and Transparency}\,$\bullet$\,
\hlaspect{aspNovelty}{Originality and Novelty}\,$\bullet$\,
\hlaspect{aspLim}{Limitations}\,$\bullet$\,
\hlaspect{aspClarity}{Clarity and Writing}\,$\bullet$\,
\hlaspect{aspImpact}{Significance and Impact}\,$\bullet$\,
\hlaspect{aspFWP}{Future Work Potential}\,$\bullet$\,
\hlaspect{aspRW}{Related Work and Positioning}\,$\bullet$\,
\hlaspect{aspEthics}{Ethics and Societal Impact}.\par\vspace{2pt}
{\footnotesize\itshape
  \emph{Ethics and Societal Impact} is solicited by NeurIPS
  through the separate Broader Impact statement,
  which is why it does not appear inside the four general criteria above.}
\end{tcolorbox}
\caption{%
Sentence-level overlay of the NeurIPS 2026 General Reviewing Guidelines onto our ten aspects.
The text reproduces the official NeurIPS prompts verbatim,
with each sub-question rendered in the color of the single aspect it activates;
the bottom panel gives the aspect-to-color legend.
Every NeurIPS sub-question maps onto exactly one aspect,
demonstrating that the ten-way decomposition is a strict, sentence-level refinement
of the official NeurIPS reviewer prompts.
\emph{Ethics and Societal Impact} does not appear inside the four general criteria
because it is solicited separately by NeurIPS through the Broader Impact statement.}
\label{fig:aspect_overlay}
\end{figure}

\paragraph{Aspect definitions and reviewer prompts.}
Table~\ref{tb:aspect_def} lists the ten aspects together with their full definitions
and the prompts used to elicit them from the LLM and human extractor.

\begin{table}[H]
\centering
\scriptsize
\begin{tabularx}{\linewidth}{@{}lY@{}}
\toprule
\textbf{Aspect} & \textbf{Definition and Prompt} \\
\midrule

\stackcell{
Claims\\Evidence\\Rigor
} &
\textit{Definition}: Correctness of proofs and adequacy of empirical support for each stated claim.

\textit{Prompt}: List the paper’s main claims and judge whether each is supported by (a) correct proofs (cite lemmas you checked) and (b) adequate empirical evidence (baselines, ablations, statistics); note any mismatches between claims and evidence. \\

\addlinespace

\stackcell{
Clarity\\Writing
}&
\textit{Definition}: Readability, organization, clear definitions, consistent notation, and comprehensible experimental descriptions.

\textit{Prompt}: Is the exposition easy to follow? For theory, are all terms/assumptions defined and symbols consistent? For experiments, are dataset/task definitions, metrics, hyperparameters, seeds, and training details specified so an expert could reproduce key results? \\

\addlinespace
\stackcell{Significance\\Impact}
 &
\textit{Definition}: Importance of the problem and expected influence on research or practice. 

\textit{Prompt}: Who would use these ideas, and how would they change research or practice? Does the work tackle an important task better than prior art or materially advance understanding? \\

\addlinespace
\stackcell{Originality\\Novelty} &
\textit{Definition}: Newness of insights, analysis, tasks, methods, or well-reasoned combinations, and clarity versus prior work. 

\textit{Prompt}: What’s genuinely new (insights, analysis, task, method, or a well-reasoned combination), and is the difference vs.\ prior work explicit and well-cited? \\

\addlinespace

\stackcell{
Related Work\\Positioning
}  &
\textit{Definition}: Accuracy and completeness of positioning relative to essential prior and concurrent work. 

\textit{Prompt}: Does the paper correctly place itself in context, citing essential prior or concurrent work and explaining concrete differences to avoid overstating novelty? \\

\addlinespace

\stackcell{
Reproducibility\\Transparency
} &
\textit{Definition}: Availability and sufficiency of artifacts and procedural detail to reproduce key results. 

\textit{Prompt}: Are artifacts/details sufficient to reproduce the main results (configs, code/data access or alternatives, hardware/compute, preprocessing)? If some assets can’t be shared, do authors provide a viable path for verification? \\

\addlinespace

\stackcell{
Experimental\\Design\\Robustness
} &
\textit{Definition}: Soundness of evaluation and stability of results to randomness, hyperparameters, and distributional changes.

\textit{Prompt}: Are evaluation protocols sound (data splits, metrics, baselines), analyses appropriate (variance, statistical tests), and results stable across seeds/hyperparameters or plausible shifts? \\

\addlinespace

Limitations &
\textit{Definition}: Explicitness and realism of limitations, assumptions, and failure modes, with reasonable mitigations. 

\textit{Prompt}: Are scope limits, assumptions, and failure modes explicit, and are proposed mitigations or cautions reasonable? What critical omissions should be added? \\

\addlinespace

\stackcell{
Ethics\\Societal\\Impact
}&
\textit{Definition}: Potential concerns, compliance, and mitigation related to bias/fairness, privacy/security, misuse, and IRB. 

\textit{Prompt}: Could the work enable harm (bias/discrimination, privacy/security, misuse, legal non-compliance)? Are risks, approvals, and mitigations documented; if not, what should be added? \\

\addlinespace

\stackcell{
Future Work\\Potential
}&
\textit{Definition}: Degree to which the work opens meaningful avenues for follow-up research or practical extensions. 

\textit{Prompt}: Does the work open promising avenues for follow-up research? \\

\bottomrule
\end{tabularx}
\caption{List of aspects used to extract features for predicting the average reviewer score of a paper.}
\label{tb:aspect_def}
\end{table}

\subsection{Double Cross-Fitting for the $A$-Adaptive Rule $R_j(A)$}
\label{appdx:double_cv}

The single-level $k$-fold cross-fitting of \Cref{sec:dataset}
(reward learned on the $-i$ folds for each train sample $i$)
is sufficient for the linear rules $R_S^{\mathrm{lin}}$ and $R_j^{\mathrm{lin}}$,
because their rewards are closed-form functions of train-fold covariances
and do not involve a learned regression on $A$.
The $A$-adaptive nonparametric rule $R_j(A)$ is different:
it requires fitting a regression of a pseudo-outcome
$\widehat\phi_{i,\{j\}}$ on $A$,
which is itself a learned object built from a first-stage estimator $\widehat m_S$.
Reusing the same $k$-fold partition for the first-stage fit
and the pseudo-outcome regression would leak first-stage residuals
into the selection-reward evaluation.

We therefore use a \emph{double} cross-fitting scheme for $R_j(A)$:
\begin{itemize}\itemsep=2pt
\item \textbf{Outer $k_{\mathrm{out}}$-fold (within the train set).}
For each outer fold,
the selection-rule regression $A \mapsto \widehat R_j(A)$
is fit on the remaining $k_{\mathrm{out}}-1$ folds
and applied to the held-out fold.
\item \textbf{Inner $k_{\mathrm{in}}$-fold (within each outer training set).}
The first-stage estimator $\widehat m_S$ is cross-fit on $k_{\mathrm{in}}$ inner folds
to produce orthogonal pseudo-outcomes $\widehat\phi_{i,\{j\}}$
that are independent of $(Y_i, A_i, H_{i,S})$ for every observation.
\end{itemize}
We use $k_{\mathrm{out}} = k_{\mathrm{in}} = 5$.
This double scheme is used \emph{only} for $R_j(A)$;
the linear rules and the importance / learned-agreement baselines
use the single-level cross-fitting of \Cref{sec:dataset}.

% \begin{figure}[H]
%     \centering
%     \includegraphics[width=0.85\linewidth]{Figures/plot_methods_2019to2021.png}
%     \caption{%
%     Cross-year generalization:
%     training on ICLR 2019 ($1{,}417$ papers)
%     and testing on ICLR 2021 ($1{,}991$ papers).
%     The three selection rules shown ---
%     importance baseline $\beta_{j,\perp}$,
%     adaptive $R_j(A)$, and linear $R_j^{\mathrm{lin}}$ ---
%     preserve the same qualitative ranking observed under the iid 80/20 split
%     (\Cref{fig:main_MAE_plots}).}
%     \label{fig:cross_year}
% \end{figure}

\subsection{Learned-Agreement Baseline}
\label{appdx:learned_agreement}

The learned-agreement rule is a deployment-friendly heuristic
that selects the aspects on which the AI and the human
are predicted to disagree most.
The intuition is that aspects with high predicted AI--human disagreement
are the ones where the AI signal is least trusted,
and where human input is therefore most informative.

\paragraph{Disagreement target.}
For each aspect $j$, we define a per-aspect disagreement score
\[
  D_j \;=\; \bigl|\,A_{j,0} - H_{j,1}\,\bigr|,
\]
the absolute difference between the LLM's score and the (proxy) human score on that aspect.
Other choices --- e.g., a logic-reason mismatch indicator,
or the squared difference --- give qualitatively similar results;
we report results with the absolute difference for simplicity.

\paragraph{Predicting disagreement from $A$ alone.}
At deployment, $H$ is not available on the aspects we have not yet queried,
so any disagreement signal used for selection must be expressible
as a function of $A$ alone.
We therefore learn a regressor
$\widehat{D}_j: \mathcal{A} \to \mathbb{R}_+$ that predicts $D_j$ from $A$:
\[
  \widehat{D}_j(A) \;=\; \widehat{\mathbb{E}}\!\left[\,D_j \,\big|\, A\,\right].
\]
We use ridge regression for $\widehat{D}_j$,
fit on the same outer training fold used by the other selection rules
(\Cref{sec:dataset}),
which guarantees that no human signal from the held-out fold
leaks into the selection.

\paragraph{Selection rule.}
Given a new $A$ at deployment,
the rule selects
\[
  \pi_{\mathrm{LA}}(A) \;=\;
  \mathrm{argtop}_{n_{\mathrm{sel}}}\,
  \bigl\{\,\widehat{D}_j(A) : j \in [J]\,\bigr\}.
\]

\paragraph{What this baseline does and does not capture.}
Learned-agreement is appealing because the disagreement target $D_j$
is intuitive, observable on the training set,
and depends only on $A$ at deployment time.
It is, however, blind to the \emph{predictive value} of resolving a disagreement:
an aspect on which AI and human disagree wildly
but neither signal correlates with the outcome $Y$
would still be selected.
The risk-based rules of \Cref{sec:selection_rules_used} weigh disagreement
by predictive relevance through the residualized linear coefficient $\beta_{j,\perp}$,
which is one explanation for their consistent improvement over learned-agreement
in our experiments (\Cref{sec:results}).

\subsection{Full Per-Model and Per-LLM Metric Table}
\label{appdx:full_metric_table}

Table~\ref{tab:average_metric_values} reports,
for each prediction model and LLM combination,
the average performance of the five selection rules averaged
over all 20 train-test splits and over $n_{\mathrm{sel}} \le 5$
(since all methods largely converge by five aspects).
For each metric and row, the top method to four decimal places is bolded.

\begin{table}[H]
    \centering
    \scriptsize
\setlength{\tabcolsep}{2.9pt}
\begin{tabular}{@{}l|lll|lll|lll|lll|lll@{}}
\multicolumn{16}{c}{Metrics: MAE ($\downarrow$), RMSE ($\downarrow$), R2 ($\uparrow$)} \\
\toprule
Regression Model / & \multicolumn{3}{c|}{$\beta_{j,\perp}$} & \multicolumn{3}{c|}{$R_j(A)$} & \multicolumn{3}{c|}{$R_j$ (A-independent)} & \multicolumn{3}{c|}{$R_S^{\text{lin}}$ (Optimal)} & \multicolumn{3}{c}{$R_j^{\text{lin}}$ (Singleton)} \\
Language Models & MAE & RMSE & R2 & MAE & RMSE & R2 & MAE & RMSE & R2 & MAE & RMSE & R2 & MAE & RMSE & R2 \\
\midrule
GB / GPT-5-mini & 0.596 & 0.749 & 0.515 & 0.544 & 0.684 & 0.592 & 0.524 & 0.660 & 0.619 & \textbf{0.523} & \textbf{0.658} & \textbf{0.621} & 0.524 & 0.660 & 0.619 \\
Rand. Forest & 0.610 & 0.766 & 0.493 & 0.550 & 0.693 & 0.582 & 0.533 & 0.671 & 0.606 & \textbf{0.533} & \textbf{0.670} & \textbf{0.607} & 0.533 & 0.671 & 0.606 \\
MLP & 0.601 & 0.757 & 0.504 & 0.554 & 0.699 & 0.574 & 0.528 & 0.666 & 0.611 & \textbf{0.526} & \textbf{0.664} & \textbf{0.614} & 0.528 & 0.666 & 0.611 \\
XGBoost & 0.598 & 0.752 & 0.511 & 0.547 & 0.688 & 0.587 & 0.529 & 0.666 & 0.611 & \textbf{0.528} & \textbf{0.665} & \textbf{0.612} & 0.529 & 0.666 & 0.611 \\
Ridge & 0.594 & 0.750 & 0.512 & 0.542 & 0.689 & 0.586 & 0.523 & 0.661 & 0.616 & \textbf{0.522} & \textbf{0.659} & \textbf{0.618} & 0.523 & 0.661 & 0.616 \\
\midrule
GB / GPT-4.1 & 0.636 & 0.804 & 0.459 & 0.543 & 0.690 & 0.593 & 0.530 & 0.670 & 0.614 & \textbf{0.523} & \textbf{0.662} & \textbf{0.622} & 0.530 & 0.670 & 0.614 \\
Rand. Forest & 0.646 & 0.820 & 0.439 & 0.550 & 0.700 & 0.582 & 0.537 & 0.679 & 0.605 & \textbf{0.531} & \textbf{0.671} & \textbf{0.612} & 0.537 & 0.679 & 0.605 \\
MLP & 0.640 & 0.808 & 0.454 & 0.550 & 0.697 & 0.586 & 0.535 & 0.673 & 0.611 & \textbf{0.527} & \textbf{0.664} & \textbf{0.620} & 0.534 & 0.673 & 0.612 \\
XGBoost & 0.637 & 0.809 & 0.452 & 0.553 & 0.704 & 0.577 & 0.541 & 0.683 & 0.599 & \textbf{0.533} & \textbf{0.675} & \textbf{0.607} & 0.540 & 0.683 & 0.599 \\
Ridge & 0.624 & 0.789 & 0.478 & 0.545 & 0.696 & 0.586 & 0.529 & 0.668 & 0.616 & \textbf{0.523} & \textbf{0.660} & \textbf{0.624} & 0.529 & 0.668 & 0.616 \\
\midrule
GB / Gemini-2.0 & 0.709 & 0.889 & 0.316 & 0.568 & 0.723 & 0.539 & 0.560 & 0.711 & 0.552 & \textbf{0.557} & \textbf{0.707} & \textbf{0.557} & 0.560 & 0.712 & 0.552 \\
Rand. Forest & 0.721 & 0.905 & 0.292 & 0.577 & 0.734 & 0.526 & 0.568 & 0.721 & 0.541 & \textbf{0.566} & \textbf{0.719} & \textbf{0.544} & 0.569 & 0.722 & 0.540 \\
MLP & 0.721 & 0.902 & 0.297 & 0.576 & 0.730 & 0.530 & 0.567 & 0.718 & 0.545 & \textbf{0.562} & \textbf{0.712} & \textbf{0.552} & 0.567 & 0.717 & 0.546 \\
XGBoost & 0.716 & 0.900 & 0.299 & 0.575 & 0.731 & 0.528 & 0.566 & 0.719 & 0.542 & \textbf{0.564} & \textbf{0.716} & \textbf{0.546} & 0.567 & 0.720 & 0.541 \\
Ridge & 0.718 & 0.952 & 0.178 & 0.579 & \textbf{0.771} & \textbf{0.454} & 0.572 & 0.783 & 0.407 & \textbf{0.568} & 0.775 & 0.420 & 0.572 & 0.784 & 0.402 \\
\bottomrule
\end{tabular}
    \caption{Average performance of different selection rules, averaged over all 20 train-test splits and over $n_{\mathrm{sel}} \le 5$. Each horizontal section corresponds to a different LLM, with different regression models. For each metric and row, the top method (to four decimal places) is bolded. ``GB'' = gradient boosting; ``Gemini-2.0'' is Gemini 2.0 Flash. The $R_j(A)$ column uses the estimator without the centering of Remark~\ref{rem:centering}; the centered estimator is evaluated for ridge and gradient boosting in Appendix~\ref{app:centered} and Table~\ref{tab:e1_mae}.}
    \label{tab:average_metric_values}
\end{table}

\subsection{Logic Reasons per Aspect}

For each aspect, the logic reasons from which the human and AI reviewers can select are displayed in the table below:

{\scriptsize
\begin{longtable}{@{}l>{\raggedright\arraybackslash}p{0.5\linewidth}@{}}
% \caption{Logic reasons for review dimensions, for each aspect.}
% \label{tab:logic-reasons} \\

\toprule
\textbf{Code} & \textbf{Logic reason} \\
\midrule
\endfirsthead

\toprule
\textbf{Code} & \textbf{Logic reason} \\
\midrule
\endhead

\midrule
\multicolumn{2}{r}{\textit{Continued on next page}} \\
\endfoot

\bottomrule
\endlastfoot

% ===================== ClaimsEvidenceRigor =====================
\multicolumn{2}{@{}l}{\textbf{ClaimsEvidenceRigor} (all codes prefixed by \texttt{rigor\_})} \\
\midrule

\texttt{supported\_theory\_and\_empirics} &
Each primary claim is backed by correct proofs and adequate empirical evidence. \\

\texttt{theory\_ok\_empirics\_incomplete} &
Theoretical proofs correct for the stated assumptions, empirical evidence incomplete. \\

\texttt{empirics\_strong\_no\_theory} &
Empirical evidence strong (baselines, statistics), but theoretical support is absent. \\

\texttt{proof\_gaps\_undermine\_claims} &
Key proofs contain gaps or unverified steps that undermine the corresponding claims. \\

\texttt{claims\_overstated\_vs\_results} &
Claim wording overstates results relative to what the experiments actually demonstrate. \\

\texttt{insufficient\_baselines\_ablations} &
Baselines or ablations are insufficient to substantiate superiority claims. \\

\texttt{statistics\_inadequate} &
Statistical analysis is inadequate to support the reported differences. \\

\texttt{details\_missing\_unverifiable} &
Evidence is not verifiable due to missing experimental details. \\

\texttt{small\_sample\_generalization\_weak} &
Results rely on small samples and are not convincingly generalizable. \\

\texttt{evidence\_well\_matched\_to\_claims} &
Claims are well-matched to the presented evidence without overreach. \\

\midrule
% ===================== ClarityWriting =====================
\multicolumn{2}{@{}l}{\textbf{ClarityWriting} (all codes prefixed by \texttt{clarity\_})} \\
\midrule

\texttt{clear\_structure\_and\_defs} &
The paper is well organized with clear definitions and consistent notation. \\

\texttt{missing\_or\_late\_definitions} &
Important terms are undefined or introduced late, hindering comprehension. \\

\texttt{ambiguous\_notation} &
Notation is inconsistent or overloaded, causing ambiguity. \\

\texttt{figures\_tables\_helpful} &
Figures and tables are informative and improve readability. \\

\texttt{dense\_exposition\_needs\_edit} &
The exposition is dense and would benefit from restructuring or illustrative examples. \\

\texttt{language\_issues\_minor} &
Language issues are minor and do not impede understanding. \\

\texttt{repro\_details\_scattered} &
Reproduction-critical details are scattered and hard to locate. \\

\texttt{appendix\_crossrefs\_effective} &
Appendix and main text are cross-referenced effectively. \\

\midrule
\multicolumn{2}{@{}l}{\textbf{SignificanceImpact} (all codes prefixed by \texttt{impact\_})} \\
\midrule
\texttt{important\_problem\_clearly\_addressed} &
The work addresses an important problem with clear impact on research or practice. \\

\texttt{incremental\_limited} &
Improvements are incremental with limited expected impact. \\

\texttt{theoretical\_insight\_advances} &
Theoretical insights substantially advance understanding of the topic. \\

\texttt{niche\_audience\_limited\_reach} &
The contribution primarily benefits a narrow niche community. \\

\texttt{potential\_to\_change\_practice} &
If validated at scale, the method could materially change practice. \\

\texttt{speculative\_claims} &
Impact claims are speculative and not well supported. \\

\texttt{foundational\_for\_future\_work} &
The work provides a strong foundation for subsequent developments. \\

\texttt{application\_relevance\_compelling} &
Application relevance is compelling and well argued. \\

\midrule
\multicolumn{2}{@{}l}{\textbf{OriginalityNovelty} (all codes prefixed by \texttt{novelty\_})} \\
\midrule

\texttt{new\_algorithm} &
This work proposes a new algorithm for a known problem and justifies its design. \\

\texttt{new\_task\_definition} &
This work defines a new problem or task with clear motivation and evaluation protocol. \\

\texttt{new\_theoretical\_framework} &
This work offers a new theoretical framework or analysis that yields fresh insights. \\

\texttt{novel\_combination\_well\_reasoned} &
This work combines existing techniques in a novel, well-reasoned way. \\

\texttt{minor\_modifications\_only} &
Novelty is limited to minor modifications or tuning. \\

\texttt{comparative\_study\_new\_insights} &
A careful comparative study yields new insights about existing methods. \\

\texttt{missing\_citations\_undermine} &
Novelty claims are undermined by missing or incomplete citations. \\

\texttt{replication\_with\_new\_findings} &
This is a replication or reproduction study with novel, informative findings. \\

\midrule
\multicolumn{2}{@{}l}{\textbf{RelatedWorkPositioning} (all codes prefixed by \texttt{rw\_})} \\
\midrule

\texttt{comprehensive\_and\_clear\_differences} &
Related work is comprehensive and differences are articulated clearly. \\

\texttt{key\_prior\_work\_missing} &
Key prior work is missing or insufficiently discussed. \\

\texttt{concurrent\_work\_acknowledged} &
Concurrent work is acknowledged and properly contrasted. \\

\texttt{novelty\_overstated\_vs\_prior} &
The paper overstates novelty relative to prior art. \\

\texttt{covers\_adjacent\_non\_ml\_literature} &
Connections to adjacent or non-ML literature are appropriately covered. \\

\texttt{mischaracterizes\_prior\_work} &
Prior work is mischaracterized or compared unfairly. \\

\texttt{survey\_helpful\_context} &
The survey of the area is accurate and helpful to readers. \\

\texttt{differences\_vague} &
Citations are present but differences remain vague. \\

\midrule
\multicolumn{2}{@{}l}{\textbf{ReproducibilityTransparency} (all codes prefixed by \texttt{repr\_})} \\
\midrule

\texttt{code\_data\_available\_documented} &
Code and data are available with sufficient documentation to reproduce results. \\

\texttt{partial\_artifacts\_enable\_partial\_repro} &
Partial artifacts are shared, enabling partial reproduction. \\

\texttt{no\_artifacts\_but\_detailed\_procedure} &
No artifacts are provided, but procedures and hyperparameters are detailed. \\

\texttt{critical\_details\_missing} &
Critical details (configs, seeds, preprocessing) are missing for reproduction. \\

\texttt{compute\_hardware\_reported} &
Compute and hardware details are reported transparently. \\

\texttt{seed\_hparam\_reporting\_inadequate} &
Seed and hyperparameter reporting is inadequate for stability checks. \\

\texttt{provenance\_and\_licensing\_documented} &
Dataset provenance and licensing are documented. \\

\texttt{restricted\_access\_with\_verification\_path} &
Data or models have restricted access but a verification path is described. \\

\midrule
\multicolumn{2}{@{}l}{\textbf{ExperimentalDesignRobustness} (all codes prefixed by \texttt{edr\_})} \\
\midrule

\texttt{sound\_protocols\_strong\_baselines} &
Evaluation uses sound protocols and strong, relevant baselines. \\

\texttt{missing\_key\_baselines} &
Key baselines are missing or outdated. \\

\texttt{poor\_metric\_choice} &
Metrics are ill-suited to the stated objectives. \\

\texttt{variance\_across\_seeds\_analyzed} &
Variance across random seeds is analyzed and acceptable. \\

\texttt{hparam\_sensitivity\_reported} &
Sensitivity to hyperparameters is reported and reasonable. \\

\texttt{ood\_shift\_robustness\_tested} &
Robustness to distribution shift is tested and encouraging. \\

\texttt{ablations\_thorough\_and\_informative} &
Ablation studies are thorough and informative. \\

\texttt{cherry\_picking\_concern} &
Results selection appears cherry-picked or under-specified. \\

\midrule
\multicolumn{2}{@{}l}{\textbf{Limitations} (all codes prefixed by \texttt{lim\_})} \\
\midrule

\texttt{explicit\_with\_mitigations} &
Limitations are explicit with reasonable mitigations. \\

\texttt{superficial\_or\_generic} &
Limitations are superficial or generic. \\

\texttt{failure\_modes\_characterized} &
Failure modes are characterized with illustrative cases. \\

\texttt{scope\_clearly\_bounded} &
Scope is clearly bounded to stated assumptions and regimes. \\

\texttt{data\_bias\_acknowledged} &
Data or sampling bias is acknowledged. \\

\texttt{compute\_constraints\_limit\_applicability} &
Compute or resource constraints limit applicability. \\

\texttt{safety\_or\_deployment\_caveats} &
Safety or deployment caveats are clearly stated. \\

\texttt{important\_limitations\_omitted} &
Important limitations are omitted. \\

\midrule
\multicolumn{2}{@{}l}{\textbf{EthicsSocietalImpact} (all codes prefixed by \texttt{eth\_})} \\
\midrule

\texttt{risks\_identified\_and\_mitigated} &
Potential risks are identified with concrete mitigations. \\

\texttt{irb\_or\_consent\_addressed} &
IRB/consent or human-subjects considerations are addressed. \\

\texttt{privacy\_security\_not\_adequate} &
Privacy or security implications are not adequately considered. \\

\texttt{bias\_fairness\_evaluated} &
Bias and fairness are evaluated with appropriate metrics. \\

\texttt{dual\_use\_discussed} &
Dual-use or misuse scenarios are discussed. \\

\texttt{licensing\_and\_legal\_compliance} &
Licensing and legal compliance are documented. \\

\texttt{significant\_concern\_flag} &
Significant ethical concerns warrant further review. \\

\texttt{minimal\_risk\_given\_setting} &
Ethical risks are minimal given the problem setting. \\

\midrule
\multicolumn{2}{@{}l}{\textbf{FutureWorkPotential} (all codes prefixed by \texttt{fwp\_})} \\
\midrule

\texttt{opens\_new\_directions} &
The work opens clear new research directions. \\

\texttt{tools\_or\_datasets\_for\_followup} &
It provides tools or datasets that enable follow-up studies. \\

\texttt{path\_to\_applications} &
There is a plausible path to impactful applications. \\

\texttt{limited\_by\_scope} &
Future work is limited by narrow scope or constraints. \\

\texttt{testable\_hypotheses\_and\_next\_steps} &
The paper articulates testable hypotheses and concrete next steps. \\

\texttt{roadmap\_for\_scaling\_or\_generalizing} &
A roadmap for scaling or generalizing is outlined. \\

\texttt{extensions\_to\_theory\_or\_benchmarks} &
Extensions to theory or benchmarks are natural and compelling. \\

\texttt{resources\_may\_limit\_progress} &
Future work relies on resources unlikely to be broadly accessible. \\

\end{longtable}
}
% =============================================================================
%  appendix/additional_analyses.tex  (new in the camera-ready)
%  Analyses run during the NeurIPS 2026 discussion period. Every table is
%  generated by bdrm_peer_review/scripts/paper_figures/build_appendix_tables.py
%  from runs/rebuttal/ (see runs/rebuttal/README.md) and \input from tables/.
% =============================================================================

\section{Additional Analyses}
\label{app:additional}

Unless stated otherwise, all analyses use the Average aggregation, the $20$ random $80/20$ splits of \Cref{sec:dataset}, the same prediction heads, and the same imputation of un-queried aspects as the main experiments. Intervals are $95\%$ normal-approximation intervals over splits ($1.96\,\hat\sigma/\sqrt{n}$); differences are paired by split, and $^{*}$ marks a difference whose interval excludes zero. No new LLM calls were made for any analysis in this section.

\subsection{Notation and rule names}
\label{app:notation}

\begin{table}[H]
\centering
\small
\caption{Notation.}
\label{tab:notation}
\begin{tabular}{@{}ll@{}}
\toprule
Symbol & Meaning \\
\midrule
$W$ & complex input (e.g., a manuscript) \\
$Y$ & outcome to be predicted (average reviewer score in \Cref{sec:peer_review}) \\
$X = (X_1,\dots,X_J)$ & pre-specified decision-relevant factors (aspects) \\
$A = (A_1,\dots,A_J)$ & AI signal on every factor; always observed \\
$H = (H_1,\dots,H_J)$ & human signal on every factor; costly, observed only on queried factors \\
$E_j$ & free-text AI rationale for factor $j$ (not used for selection or prediction) \\
$Z = f(A)$ & context available to the policy ($Z$ constant: non-adaptive; $Z = A$: adaptive) \\
$\pi$, $S = \pi(Z)$ & selection policy and the realized queried subset \\
$n_{\mathrm{sel}}$ & human budget, $|\pi(Z)| \le n_{\mathrm{sel}}$ \\
$g_\theta(A, H_S)$ & downstream predictor (an evaluation device in our experiments) \\
$R_S(z)$ & reward of subset $S$ at context $z$, \eqref{eq:Rpi_def} \\
$H_{S,\perp}$, $\beta_{S,\perp}$ & human signal residualized on $A$ and its coefficient, \Cref{sec:linear_insight} \\
$\widehat\phi_{i,S}$ & orthogonal pseudo-outcome of the two-stage estimator, \eqref{eq:phi_pi_orth} \\
\bottomrule
\end{tabular}
\end{table}

\begin{table}[H]
\centering
\small
\caption{Selection rules, with the single name used in every figure and table.}
\label{tab:rule_names}
\begin{tabular}{@{}cll>{\raggedright\arraybackslash}p{5.2cm}@{}}
\toprule
No. & Name & Type & Criterion \\
\midrule
1 & $R_j(A)$ & nonparametric, adaptive & two-stage estimate of $R_{\{j\}}(A)$, centered (\Cref{rem:centering}) \\
2 & $R_j$ & nonparametric, non-adaptive & average pseudo-outcome \eqref{eq:Rpi_nonparam_nonadaptive} \\
3 & $R_S^{\mathrm{lin}}$ & linear, subset & closed form \eqref{eq:Rpi_linear_decomp}, all subsets of size $n_{\mathrm{sel}}$ \\
4 & $R_j^{\mathrm{lin}}$ & linear, singleton & $\|\beta_{j,\perp}\|^2\,\Var(H_{j,\perp})$ \\
5 & $\|\beta_{j,\perp}\|$ & ablation & relevance only \\
6 & Learned agreement & heuristic & predicted $|A_{j,0} - H_{j,1}|$ (Appendix~\ref{appdx:learned_agreement}) \\
7 & Fixed subset & non-learned & Claims $\to$ Originality $\to$ Significance $\to$ Experimental Design \\
8 & Random-$k$ & non-learned & $n_{\mathrm{sel}}$ aspects uniformly at random per paper \\
\bottomrule
\end{tabular}
\end{table}

\subsection{Single-reviewer and cross-year results, and aspect rewards}
\label{app:extra_figures}

\begin{figure}[H]
    \centering
    \includegraphics[width=\linewidth]{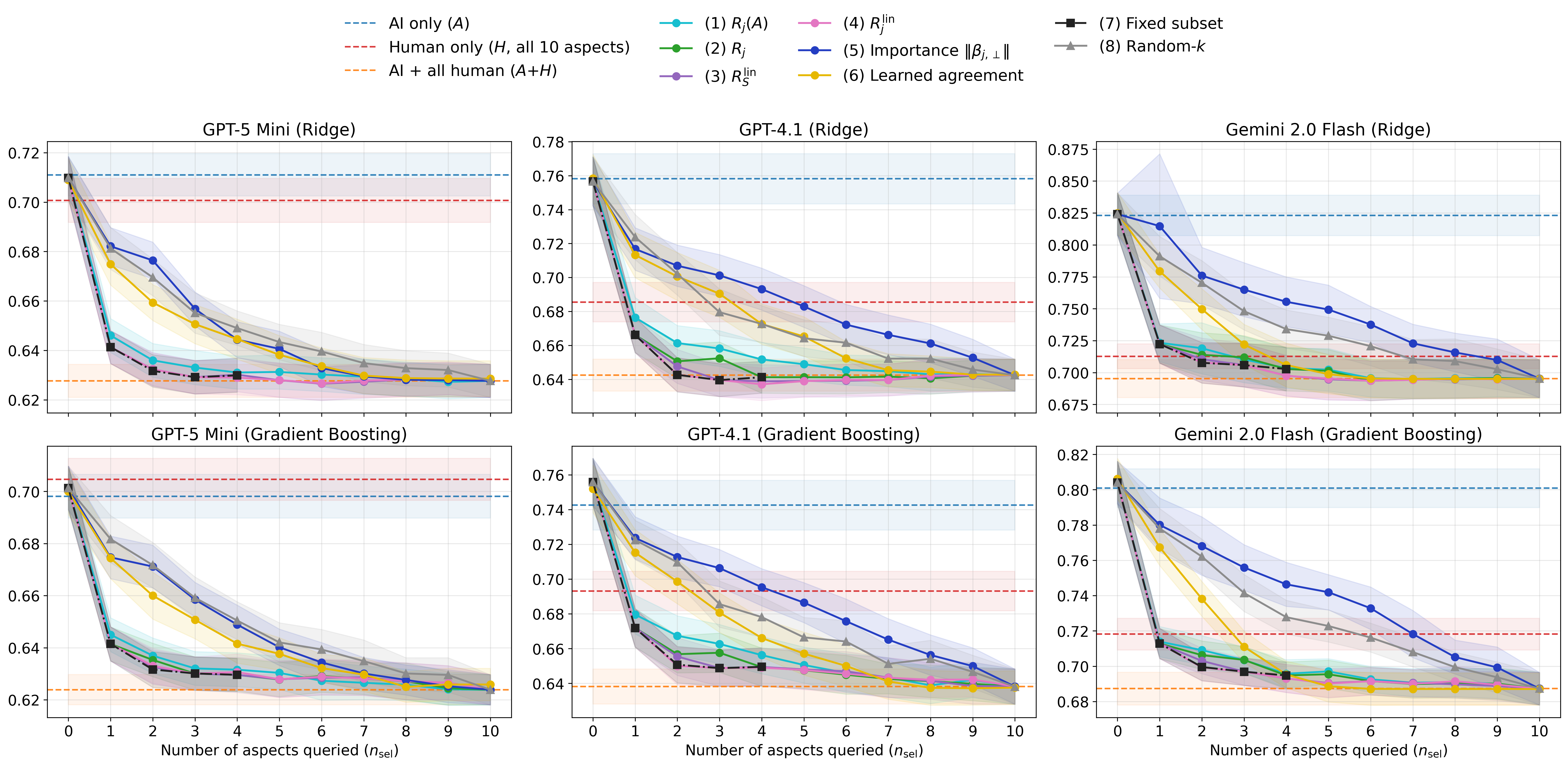}
    \caption{%
   Same as \Cref{fig:main_MAE_plots} (\Cref{sec:results}) under the Random aggregation, where $H$ is taken from a single randomly chosen reviewer. The human-only baseline degrades, while gains from combining $A$ with selectively queried $H$ become larger.}
    \label{fig:main_MAE_FIRST}
\end{figure}

\begin{figure}[H]
    \centering
    \includegraphics[width=0.9\linewidth]{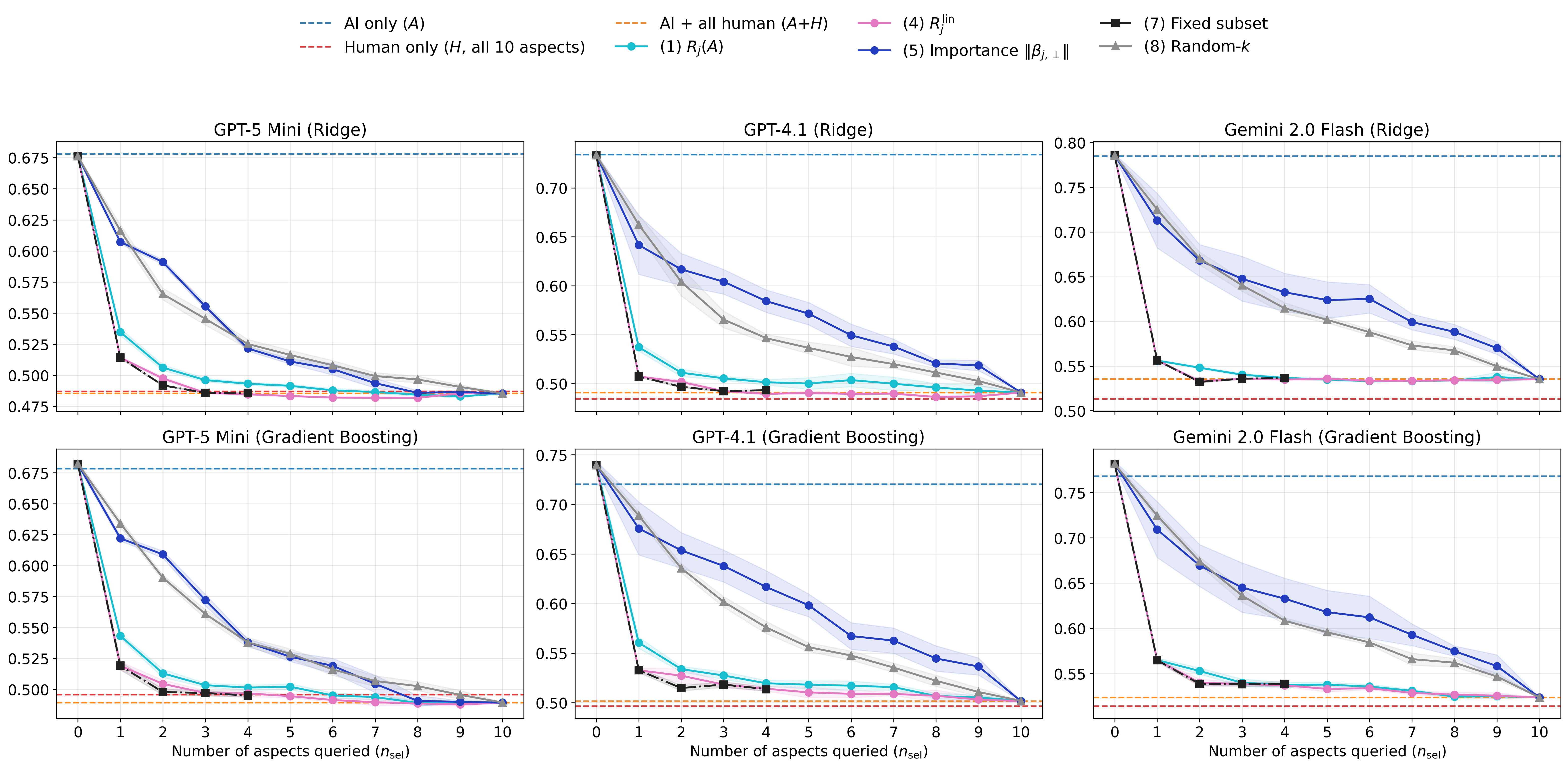}
    \caption{%
    Cross-year generalization:
    training on ICLR 2019
                and testing on ICLR 2021 (10 seeds). Rule numbers follow \Cref{sec:selection_rules_used}.}
    \label{fig:cross_year}
\end{figure}

\begin{figure}[H]
    \centering
\includegraphics[width=0.6\linewidth]{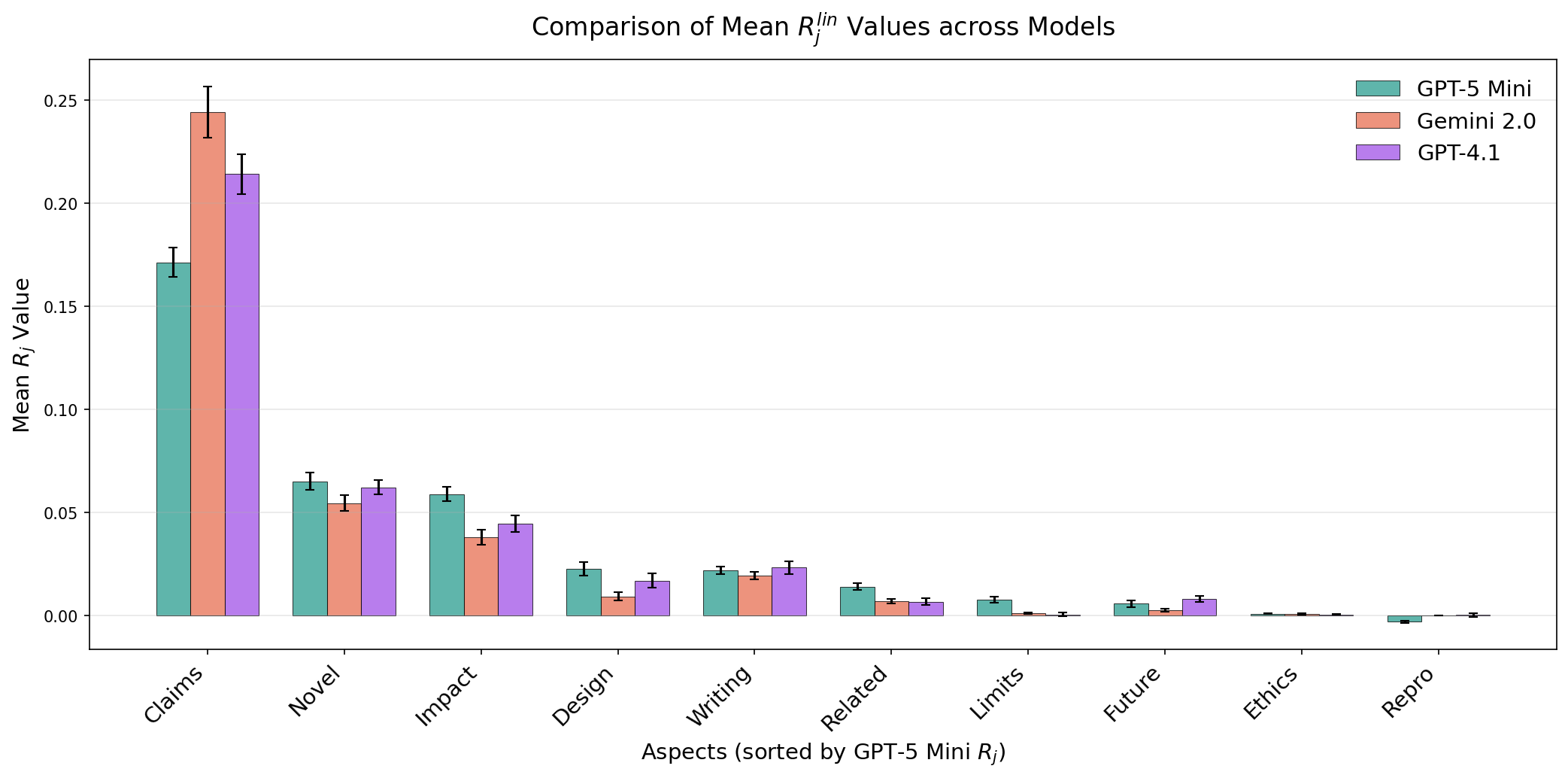}
    \caption{Mean linear singleton reward $R_j^{\mathrm{lin}}$ for each aspect and LLM, with error bars. \emph{Claims and Evidence Rigor} dominates, followed by \emph{Originality and Novelty} and \emph{Significance and Impact}.}
    \label{fig:ranking}
\end{figure}

\subsection{Fixed-subset and random baselines}
\label{app:fixed_baselines}

\Cref{tab:e1_mae} reports test MAE for all eight rules at $n_{\mathrm{sel}} \in \{2,3\}$. \Cref{tab:e1_delta} gives the paired difference between the learned rule $R_j^{\mathrm{lin}}$ and the fixed subset. At $n_{\mathrm{sel}}=3$ the two rules select the same aspects for every paper, and the difference is identically zero in all twelve configurations. At $n_{\mathrm{sel}}=2$, $R_j^{\mathrm{lin}}$ selects \emph{Claims and Evidence Rigor} for every paper and \emph{Originality and Novelty} for $77.8\%$ (GPT-5-mini), $95.2\%$ (GPT-4.1), and $96.8\%$ (Gemini-2.0) of paper--split pairs; the remainder receive \emph{Significance and Impact}. Across splits, three LLMs, and both heads, the learned rule is never significantly better than the fixed subset, and it is slightly worse in some configurations, most visibly on the 2019$\to$2021 split, where a rule fitted on 2019 alone escalates Significance more often. The rules that do not use the relevance--uncertainty criterion are substantially worse: random-$k$ by $0.08$--$0.17$ MAE at $n_{\mathrm{sel}} \in \{2,3\}$, and importance-only is significantly worse than random-$k$ in $11$ of $18$ configurations ($n_{\mathrm{sel}} \le 3$, 20 random splits).

\begin{table}[H]
\centering
\scriptsize
\setlength{\tabcolsep}{3pt}
\caption{Test MAE (mean $\pm$ 95\% CI over 20 random splits) for the eight selection rules.}
\label{tab:e1_mae}
\begin{adjustbox}{max width=\linewidth}
\begin{tabular}{@{}lcccccc@{}}
\toprule
Rule & GPT-5-mini/Ridge & GPT-5-mini/GB & GPT-4.1/Ridge & GPT-4.1/GB & Gemini-2.0/Ridge & Gemini-2.0/GB \\
\midrule
\multicolumn{7}{l}{\emph{$n_{\mathrm{sel}} = 2$}} \\
(1) $R_j(A)$ & $0.498\pm0.005$ & $0.502\pm0.006$ & $0.503\pm0.010$ & $0.504\pm0.007$ & $0.533\pm0.013$ & $0.525\pm0.007$ \\
(2) $R_j$ & $0.489\pm0.005$ & $0.491\pm0.007$ & $0.485\pm0.008$ & $0.485\pm0.008$ & $0.525\pm0.015$ & $0.513\pm0.007$ \\
(3) $R_S^{\mathrm{lin}}$ & $0.487\pm0.007$ & $0.492\pm0.007$ & $0.480\pm0.008$ & $0.482\pm0.008$ & $0.521\pm0.015$ & $0.513\pm0.007$ \\
(4) $R_j^{\mathrm{lin}}$ & $0.486\pm0.006$ & $0.491\pm0.006$ & $0.476\pm0.007$ & $0.478\pm0.007$ & $0.517\pm0.014$ & $0.509\pm0.008$ \\
(5) $\|\beta_{j,\perp}\|$ & $0.616\pm0.007$ & $0.615\pm0.006$ & $0.619\pm0.019$ & $0.637\pm0.017$ & $0.708\pm0.035$ & $0.708\pm0.028$ \\
(6) Learned agreement & $0.564\pm0.007$ & $0.574\pm0.009$ & $0.607\pm0.014$ & $0.613\pm0.010$ & $0.615\pm0.014$ & $0.601\pm0.010$ \\
(7) Fixed subset & $0.484\pm0.006$ & $0.487\pm0.006$ & $0.476\pm0.007$ & $0.478\pm0.007$ & $0.516\pm0.013$ & $0.509\pm0.008$ \\
(8) Random-$k$ & $0.580\pm0.005$ & $0.607\pm0.008$ & $0.591\pm0.012$ & $0.628\pm0.009$ & $0.672\pm0.011$ & $0.676\pm0.008$ \\
\midrule
\multicolumn{7}{l}{\emph{$n_{\mathrm{sel}} = 3$}} \\
(1) $R_j(A)$ & $0.485\pm0.006$ & $0.490\pm0.006$ & $0.493\pm0.010$ & $0.489\pm0.007$ & $0.522\pm0.014$ & $0.510\pm0.007$ \\
(2) $R_j$ & $0.475\pm0.006$ & $0.479\pm0.006$ & $0.486\pm0.008$ & $0.485\pm0.007$ & $0.521\pm0.014$ & $0.510\pm0.007$ \\
(3) $R_S^{\mathrm{lin}}$ & $0.475\pm0.006$ & $0.479\pm0.006$ & $0.467\pm0.008$ & $0.471\pm0.008$ & $0.510\pm0.014$ & $0.497\pm0.006$ \\
(4) $R_j^{\mathrm{lin}}$ & $0.475\pm0.006$ & $0.479\pm0.006$ & $0.467\pm0.008$ & $0.471\pm0.008$ & $0.510\pm0.014$ & $0.497\pm0.006$ \\
(5) $\|\beta_{j,\perp}\|$ & $0.564\pm0.009$ & $0.570\pm0.010$ & $0.607\pm0.011$ & $0.620\pm0.013$ & $0.682\pm0.033$ & $0.675\pm0.025$ \\
(6) Learned agreement & $0.545\pm0.007$ & $0.552\pm0.009$ & $0.561\pm0.011$ & $0.569\pm0.009$ & $0.555\pm0.015$ & $0.537\pm0.009$ \\
(7) Fixed subset & $0.475\pm0.006$ & $0.479\pm0.006$ & $0.467\pm0.008$ & $0.471\pm0.008$ & $0.510\pm0.014$ & $0.497\pm0.006$ \\
(8) Random-$k$ & $0.551\pm0.007$ & $0.570\pm0.008$ & $0.557\pm0.011$ & $0.589\pm0.012$ & $0.632\pm0.009$ & $0.632\pm0.009$ \\
\midrule
Human only ($H$) & $0.481\pm0.006$ & $0.488\pm0.006$ & $0.466\pm0.008$ & $0.467\pm0.008$ & $0.497\pm0.007$ & $0.491\pm0.006$ \\
AI only ($A$) & $0.711\pm0.009$ & $0.698\pm0.008$ & $0.758\pm0.015$ & $0.742\pm0.014$ & $0.823\pm0.016$ & $0.801\pm0.011$ \\
\bottomrule
\end{tabular}

\end{adjustbox}
\end{table}

\begin{table}[H]
\centering
\scriptsize
\caption{Paired difference in test MAE, $R_j^{\mathrm{lin}}$ minus fixed subset. Positive values favor the fixed subset.}
\label{tab:e1_delta}
\begin{adjustbox}{max width=\linewidth}
\begin{tabular}{@{}lcccccc@{}}
\toprule
Split & GPT-5-mini/Ridge & GPT-5-mini/GB & GPT-4.1/Ridge & GPT-4.1/GB & Gemini-2.0/Ridge & Gemini-2.0/GB \\
\midrule
20 random splits, $n_{\mathrm{sel}}=2$ & $+0.0022$ & $+0.0042^{*}$ & $+0.0001$ & $0$ & $+0.0003$ & $-0.0004$ \\
20 random splits, $n_{\mathrm{sel}}=3$ & $0$ & $0$ & $0$ & $0$ & $0$ & $0$ \\
2019$\to$2021, $n_{\mathrm{sel}}=2$ & $+0.0055^{*}$ & $+0.0066^{*}$ & $+0.0050^{*}$ & $+0.0125^{*}$ & $+0.0007$ & $+0.0015$ \\
2019$\to$2021, $n_{\mathrm{sel}}=3$ & $0$ & $0$ & $0$ & $0$ & $0$ & $0$ \\
\bottomrule
\end{tabular}

\end{adjustbox}
\end{table}

\subsection{Cross-reviewer holdout target}
\label{app:loo}

\paragraph{Design.} Each row is a (paper, review $r$) pair. The features are the paper-level AI signal $A$ and the selected human aspects extracted from review $r$ only. The target is the mean score of the other $n-1$ reviewers of the paper,
$Y_{\mathrm{LOO}} = (n\,\widehat Y - s_r)/(n-1)$, where $s_r$ is the score of review $r$. The text that produces $H$ therefore carries no information about the target by construction. Every paper has $n \ge 2$ reviews. Splits are the paper-level splits of \Cref{sec:dataset}, so all reviews of a paper fall on the same side. The rewards of $R_j^{\mathrm{lin}}$ are refitted for each target with cross-fitting grouped by paper. Because no learned imputation exists at the review level, un-queried aspects use the AI substitution $(H_{j,0},H_{j,1},H_{j,2}) = (1, A_{j,0}, A_{j,1})$ for both targets. As a reference, the identical rows and features are run against the standard target $\widehat Y$.

\paragraph{Results.} Absolute MAE rises for every feature set under the holdout target, because both the target and the single-review $H$ are noisier. We therefore report the fraction of the gap between the AI-only and the full $A{+}H$ MAE that the selected aspects recover (\Cref{tab:e2}). Under the holdout target, a single review's $H$ alone predicts worse than $A$ alone in five of the six configurations (e.g., MAE $0.854$ versus $0.771$ for GPT-5-mini with ridge). The absolute gap between $A$ alone and $A{+}H$ is also small under this target ($0.016$--$0.043$ MAE), so the recovered fractions should be read together with the absolute MAEs in \Cref{tab:e2}. The denominator therefore uses the $A{+}H$ endpoint of the same curve rather than $H$ alone.

\begin{table}[H]
\centering
\scriptsize
\caption{Gap recovery $(\mathrm{MAE}_A - \mathrm{MAE}_{\mathrm{selected}})/(\mathrm{MAE}_A - \mathrm{MAE}_{A+H})$ for $R_j^{\mathrm{lin}}$, standard target $\widehat Y$ versus cross-reviewer holdout target $Y_{\mathrm{LOO}}$ (ratio of means over 20 splits).}
\label{tab:e2}
\begin{adjustbox}{max width=\linewidth}
\begin{tabular}{@{}lcccccc@{}}
\toprule
 & \multicolumn{2}{c}{AI-only MAE} & \multicolumn{2}{c}{Gap recovered, $n_{\mathrm{sel}}=2$} & \multicolumn{2}{c}{Gap recovered, $n_{\mathrm{sel}}=3$} \\
\cmidrule(lr){2-3}\cmidrule(lr){4-5}\cmidrule(lr){6-7}
Configuration & standard & holdout & standard & holdout & standard & holdout \\
\midrule
GPT-5-mini/Ridge & 0.704 & 0.771 & 94.6\% & 94.3\% & 97.4\% & 94.5\% \\
GPT-5-mini/GB & 0.692 & 0.759 & 95.0\% & 86.3\% & 97.8\% & 88.7\% \\
GPT-4.1/Ridge & 0.751 & 0.814 & 95.8\% & 90.9\% & 96.5\% & 90.4\% \\
GPT-4.1/GB & 0.737 & 0.800 & 95.4\% & 91.3\% & 96.1\% & 90.6\% \\
Gemini-2.0/Ridge & 0.820 & 0.876 & 92.1\% & 86.3\% & 93.5\% & 92.3\% \\
Gemini-2.0/GB & 0.794 & 0.854 & 92.4\% & 89.1\% & 94.8\% & 92.7\% \\
\bottomrule
\end{tabular}

\end{adjustbox}
\end{table}

\subsection{Accept/reject decisions}
\label{app:accept_reject}

\paragraph{Recovered ICLR 2019 decisions.} In ICLR 2019 the final decision is recorded in the meta-review (\emph{Accept (Oral)}, \emph{Accept (Poster)}, \emph{Reject}) rather than in a separate decision note. Our original ingestion missed this field and recorded every 2019 paper as rejected. We re-fetched the decisions from the meta-reviews for all $1{,}417$ papers: $502$ are accepted ($35.4\%$; $24$ oral and $478$ poster), all $43$ papers with average score at least $7.5$ are accepted, and $204$ of the $205$ papers with average score at most $4$ are rejected. The main experiments predict the average score and are unaffected. The autonomous-LLM counts in \Cref{sec:results} and the analysis below use the recovered labels.

\paragraph{Results.} \Cref{tab:e3} reports test AUC for predicting the binary decision with logistic regression and gradient-boosting classifiers. Ranking papers by the true average reviewer score gives AUC $0.955$--$0.956$, which is the effective ceiling, since decisions near the acceptance boundary are partly idiosyncratic. Two or three selected aspects plus $A$ match or exceed all ten human aspects alone. Results restricted to ICLR 2021, whose labels were never in question, are consistent (e.g., AUC $0.904$ versus $0.872$ for GPT-5-mini with logistic regression, $n_{\mathrm{sel}}=2$ versus $H$ alone).

\begin{table}[H]
\centering
\scriptsize
\caption{Test AUC for accept/reject prediction, ICLR 2019 and 2021 (mean over 20 random splits; 95\% CIs are at most $\pm0.009$).}
\label{tab:e3}
\begin{adjustbox}{max width=\linewidth}
\begin{tabular}{@{}lcccccc@{}}
\toprule
 & \multicolumn{2}{c}{GPT-5-mini} & \multicolumn{2}{c}{GPT-4.1} & \multicolumn{2}{c}{Gemini-2.0} \\
\cmidrule(lr){2-3}\cmidrule(lr){4-5}\cmidrule(lr){6-7}
Feature set & logistic & GB & logistic & GB & logistic & GB \\
\midrule
$\widehat Y$ alone (ceiling) & 0.955 & 0.955 & 0.956 & 0.956 & 0.955 & 0.955 \\
AI only ($A$) & 0.800 & 0.810 & 0.762 & 0.772 & 0.665 & 0.697 \\
$R_j^{\mathrm{lin}}$, $n_{\mathrm{sel}}=2$ & 0.900 & 0.899 & 0.894 & 0.899 & 0.865 & 0.874 \\
$R_j^{\mathrm{lin}}$, $n_{\mathrm{sel}}=3$ & 0.906 & 0.904 & 0.902 & 0.905 & 0.871 & 0.882 \\
Fixed subset, $n_{\mathrm{sel}}=2$ & 0.900 & 0.899 & 0.894 & 0.899 & 0.865 & 0.874 \\
Human only ($H$) & 0.875 & 0.871 & 0.882 & 0.880 & 0.865 & 0.864 \\
$A$+$H$ (all aspects) & 0.906 & 0.906 & 0.900 & 0.909 & 0.872 & 0.882 \\
\bottomrule
\end{tabular}

\end{adjustbox}
\end{table}

\subsection{Context-stratified linear rule}
\label{app:stratified}

On our data a fully adaptive rule has high estimation variance, whereas a coarse categorical context can be estimated reliably. We therefore test the minimal form of adaptivity: the closed-form linear reward is computed separately within each stratum $c$ of a context $Z$ on the training folds and shrunk toward the pooled estimate, $R_j^{(\lambda)}(c) = \lambda R_j^{(c)} + (1-\lambda) R_j^{\mathrm{pooled}}$. At test time, each paper uses the reward of its stratum.

\paragraph{Contexts.} The primary context is paper type, computed from the manuscript text alone. Let $n_{\mathrm{th}}$ count the tokens \emph{theorem, lemma, proposition, corollary, proof} and $n_{\mathrm{emp}}$ count \emph{experiment, dataset, baseline, benchmark}. A paper is \emph{empirical} if $n_{\mathrm{th}} = 0$ ($1{,}801$ papers), \emph{theory} if $n_{\mathrm{th}}/(n_{\mathrm{th}}+n_{\mathrm{emp}}) \ge 0.3$ ($873$), and \emph{mixed} otherwise ($734$). A second, independent context comes from author-supplied OpenReview keywords, collapsed to \emph{theory}, \emph{applications}, and \emph{other} ($533$ / $1{,}102$ / $1{,}773$ papers). Neither construction uses scores, reviews, or decisions.

\paragraph{Results.} At $n_{\mathrm{sel}}=2$, with $\lambda = 0.5$, empirical and mixed papers escalate \{Claims, Originality\} in essentially every split for all three LLMs. Theory papers escalate \{Claims, Significance\} in $20/20$ splits for GPT-5-mini and $14/20$ for Gemini-2.0, but only $3/20$ for GPT-4.1. The keyword-based context reproduces the same switch for theory papers. At $n_{\mathrm{sel}} \in \{1,3\}$ the strata almost always select the same set; with $\lambda = 0.5$ the effect of stratification there is negligible. \Cref{tab:e4} reports the paired effect at $n_{\mathrm{sel}}=2$. Stratification improves on the pooled rule for GPT-5-mini (random splits) and for GPT-4.1 (2019$\to$2021), by $0.004$--$0.011$ MAE, but with $\lambda = 1$ it hurts GPT-4.1 on the random splits. It does not reliably improve on the fixed subset: across the $48$ paired comparisons with $n_{\mathrm{sel}} \in \{2,3\}$, $\lambda \in \{0.5, 1\}$, and both split types, it is significantly better in $2$ and significantly worse in $12$. We read this as adaptivity at the resolution the data support: the escalation pattern differs interpretably across paper types, but at this sample size and aspect resolution, the gain over a well-chosen static subset is within noise.

\begin{table}[H]
\centering
\scriptsize
\caption{Paired difference in test MAE at $n_{\mathrm{sel}}=2$, stratified (paper-type) linear rule minus the pooled linear rule or the fixed subset. Negative values favor the stratified rule.}
\label{tab:e4}
\begin{adjustbox}{max width=\linewidth}
\begin{tabular}{@{}lcccc@{}}
\toprule
 & \multicolumn{2}{c}{$\lambda=0.5$} & \multicolumn{2}{c}{$\lambda=1$} \\
\cmidrule(lr){2-3}\cmidrule(lr){4-5}
Configuration & vs.\ pooled & vs.\ fixed & vs.\ pooled & vs.\ fixed \\
\midrule
GPT-5-mini/Ridge (random) & $-0.0041^{*}$ & $-0.0020$ & $-0.0051^{*}$ & $-0.0030^{*}$ \\
GPT-5-mini/GB (random) & $-0.0045^{*}$ & $-0.0003$ & $-0.0052^{*}$ & $-0.0010$ \\
GPT-4.1/Ridge (random) & $+0.0001$ & $+0.0002$ & $+0.0029^{*}$ & $+0.0030^{*}$ \\
GPT-4.1/GB (random) & $0$ & $0$ & $+0.0062^{*}$ & $+0.0062^{*}$ \\
Gemini-2.0/Ridge (random) & $0$ & $+0.0004$ & $-0.0006$ & $-0.0003$ \\
Gemini-2.0/GB (random) & $-0.0004$ & $-0.0008$ & $0$ & $-0.0004$ \\
GPT-5-mini/Ridge (2019$\to$2021) & $-0.0007$ & $+0.0048^{*}$ & $-0.0006$ & $+0.0049^{*}$ \\
GPT-5-mini/GB (2019$\to$2021) & $-0.0010$ & $+0.0056^{*}$ & $-0.0011$ & $+0.0056^{*}$ \\
GPT-4.1/Ridge (2019$\to$2021) & $-0.0054^{*}$ & $-0.0004$ & $-0.0057^{*}$ & $-0.0007$ \\
GPT-4.1/GB (2019$\to$2021) & $-0.0101^{*}$ & $+0.0024^{*}$ & $-0.0110^{*}$ & $+0.0015$ \\
Gemini-2.0/Ridge (2019$\to$2021) & $+0.0002$ & $+0.0009$ & $+0.0003$ & $+0.0010$ \\
Gemini-2.0/GB (2019$\to$2021) & $-0.0033$ & $-0.0019^{*}$ & $-0.0031$ & $-0.0016$ \\
\bottomrule
\end{tabular}

\end{adjustbox}
\end{table}

\subsection{Protected factors}
\label{app:protected}

A protected factor is designated in advance as always requiring human evaluation and removed from the selection pool; the policy optimizes only over the remaining, delegable factors. We protect \emph{Ethics and Societal Impact}: at budget $n_{\mathrm{sel}}$, the protected rule queries ethics plus the $n_{\mathrm{sel}}-1$ aspects with the largest $R_j^{\mathrm{lin}}$ among the rest. \Cref{tab:ee1} compares it with the unprotected rule at the same total budget and with the unprotected rule at budget $n_{\mathrm{sel}}-1$. Protection costs exactly one query: its MAE matches the unprotected rule with one fewer freely selected factor, to within $0.0071$ across all $36$ configurations (3 LLMs $\times$ 2 heads $\times$ 2 split types $\times$ $n_{\mathrm{sel}} \in \{2,3,4\}$). The guarantee it buys cannot be expressed by the prediction objective.

\begin{table}[H]
\centering
\scriptsize
\setlength{\tabcolsep}{2.5pt}
\caption{Test MAE with \emph{Ethics and Societal Impact} protected (20 random splits). ``free'': unprotected $R_j^{\mathrm{lin}}$ at the same budget; ``free, $n_{\mathrm{sel}}{-}1$'': unprotected at one fewer query.}
\label{tab:ee1}
\begin{adjustbox}{max width=\linewidth}
\begin{tabular}{@{}lccccccccc@{}}
\toprule
 & \multicolumn{3}{c}{$n_{\mathrm{sel}}=2$} & \multicolumn{3}{c}{$n_{\mathrm{sel}}=3$} & \multicolumn{3}{c}{$n_{\mathrm{sel}}=4$} \\
\cmidrule(lr){2-4}\cmidrule(lr){5-7}\cmidrule(lr){8-10}
 & free & protected & free, $n_{\mathrm{sel}}{-}1$ & free & protected & free, $n_{\mathrm{sel}}{-}1$ & free & protected & free, $n_{\mathrm{sel}}{-}1$ \\
\midrule
GPT-5-mini/Ridge & $0.486\pm0.006$ & $0.520\pm0.006$ & $0.520\pm0.006$ & $0.475\pm0.006$ & $0.487\pm0.006$ & $0.486\pm0.006$ & $0.474\pm0.007$ & $0.475\pm0.006$ & $0.475\pm0.006$ \\
GPT-5-mini/GB & $0.491\pm0.006$ & $0.520\pm0.006$ & $0.520\pm0.006$ & $0.479\pm0.006$ & $0.491\pm0.007$ & $0.491\pm0.006$ & $0.478\pm0.007$ & $0.478\pm0.007$ & $0.479\pm0.006$ \\
GPT-4.1/Ridge & $0.476\pm0.007$ & $0.514\pm0.008$ & $0.514\pm0.008$ & $0.467\pm0.008$ & $0.476\pm0.007$ & $0.476\pm0.007$ & $0.463\pm0.007$ & $0.467\pm0.008$ & $0.467\pm0.008$ \\
GPT-4.1/GB & $0.478\pm0.007$ & $0.515\pm0.007$ & $0.515\pm0.007$ & $0.471\pm0.008$ & $0.478\pm0.007$ & $0.478\pm0.007$ & $0.462\pm0.007$ & $0.471\pm0.008$ & $0.471\pm0.008$ \\
Gemini-2.0/Ridge & $0.517\pm0.014$ & $0.552\pm0.012$ & $0.551\pm0.012$ & $0.510\pm0.014$ & $0.518\pm0.013$ & $0.517\pm0.014$ & $0.504\pm0.012$ & $0.511\pm0.014$ & $0.510\pm0.014$ \\
Gemini-2.0/GB & $0.509\pm0.008$ & $0.548\pm0.008$ & $0.549\pm0.008$ & $0.497\pm0.006$ & $0.509\pm0.008$ & $0.509\pm0.008$ & $0.493\pm0.006$ & $0.497\pm0.006$ & $0.497\pm0.006$ \\
\bottomrule
\end{tabular}

\end{adjustbox}
\end{table}

\subsection{Centered estimator for the adaptive rule on real data}
\label{app:centered}

\Cref{tab:e6} compares the adaptive rule $R_j(A)$ with and without the centering of \Cref{rem:centering}, recomputed with the same reward pipeline, splits, and heads (the uncentered re-run differs slightly from the published $R_j(A)$ curves because the rewards are refitted). Centering is significantly better in $20$ and significantly worse in none of $36$ configurations (3 LLMs $\times$ 2 heads $\times$ 2 split types $\times$ $n_{\mathrm{sel}} \le 3$). It does not change the ordering of rules: the centered $R_j(A)$ remains significantly worse than the fixed subset in $16$ of the $18$ random-split configurations. The adaptive rule therefore does not underperform because of an overlooked estimation defect: once the defect is corrected, the data still favor a static rule.

\begin{table}[H]
\centering
\scriptsize
\setlength{\tabcolsep}{3pt}
\caption{Adaptive rule $R_j(A)$, uncentered versus centered pseudo-outcomes: test MAE and paired difference (centered minus uncentered), 20 random splits.}
\label{tab:e6}
\begin{adjustbox}{max width=\linewidth}
\begin{tabular}{@{}lccccccccc@{}}
\toprule
 & \multicolumn{3}{c}{$n_{\mathrm{sel}}=1$} & \multicolumn{3}{c}{$n_{\mathrm{sel}}=2$} & \multicolumn{3}{c}{$n_{\mathrm{sel}}=3$} \\
\cmidrule(lr){2-4}\cmidrule(lr){5-7}\cmidrule(lr){8-10}
 & uncent. & centered & $\Delta$ & uncent. & centered & $\Delta$ & uncent. & centered & $\Delta$ \\
\midrule
GPT-5-mini/Ridge & $0.538$ & $0.528$ & $-0.0105^{*}$ & $0.508$ & $0.498$ & $-0.0098^{*}$ & $0.489$ & $0.485$ & $-0.0039^{*}$ \\
GPT-5-mini/GB & $0.548$ & $0.533$ & $-0.0145^{*}$ & $0.513$ & $0.502$ & $-0.0106^{*}$ & $0.492$ & $0.490$ & $-0.0019$ \\
GPT-4.1/Ridge & $0.527$ & $0.532$ & $+0.0049$ & $0.499$ & $0.503$ & $+0.0043$ & $0.494$ & $0.493$ & $-0.0005$ \\
GPT-4.1/GB & $0.532$ & $0.536$ & $+0.0038$ & $0.500$ & $0.504$ & $+0.0033$ & $0.487$ & $0.489$ & $+0.0017$ \\
Gemini-2.0/Ridge & $0.566$ & $0.552$ & $-0.0143^{*}$ & $0.543$ & $0.533$ & $-0.0099^{*}$ & $0.527$ & $0.522$ & $-0.0049^{*}$ \\
Gemini-2.0/GB & $0.561$ & $0.548$ & $-0.0129^{*}$ & $0.534$ & $0.525$ & $-0.0084^{*}$ & $0.516$ & $0.510$ & $-0.0064^{*}$ \\
\bottomrule
\end{tabular}

\end{adjustbox}
\end{table}

\subsection{Simulation: when does adaptivity pay?}
\label{app:simulation}

The peer-review data support a static rule. To test whether this reflects the method or the data, we simulate a setting with genuine regime structure and run the same estimators at comparable sample sizes.

\paragraph{Data-generating process.} Each instance has a latent regime $C \sim \mathrm{Uniform}\{1,2,3\}$ (think paper type) and $J = 8$ factors $X \mid C \sim \mathcal N(\mu_C, \Sigma)$, with $\Sigma_{ij} = 0.3^{|i-j|}$. The outcome is $Y = \beta^\top X + \varepsilon$ with $\beta = (1.6, 1.6, 1.6, 1.1, 1.1, 1.1, 0.4, 0.4)$ and $\varepsilon \sim \mathcal N(0, 0.7^2)$. The human signal is unbiased and regime-independent, $H_j = X_j + \nu_j$ with $\nu_j \sim \mathcal N(0, 0.3^2)$. The AI signal is biased and regime-dependent, $A_j = a_{C,j} + b_{C,j} X_j + \eta_j$. In regime $c$ the AI is uninformative ($b = 0$) on one high-relevance factor $j^\star(c) \in \{1,2,3\}$ and weak ($b = 0.25$) on a second factor $\{4,5,6\}$; the factor differs across regimes. A one-hot encoding of $C$ is appended to $A$, so the regime is recoverable from the AI signal. The oracle adaptive rule queries $j^\star(c)$ in regime $c$, and the best non-adaptive rule must compromise. With the oracle Bayes predictor, the gap between the two is $22.7\%$ ($n_{\mathrm{sel}}=1$) and $22.6\%$ ($n_{\mathrm{sel}}=2$) of the gap between AI only and full human review. We use $50$ replicates per training-set size and a fixed test set of $25{,}000$ points. Rules and ridge predictors are estimated on each training sample.

\paragraph{Results.} \Cref{fig:e5} and \Cref{tab:e5} show test MSE against the training-set size $n$. The non-adaptive linear rule wins at the smallest sample sizes purely on estimation variance. The stratified linear rule, which observes $C$, overtakes it from $n \approx 125$--$250$. The original two-stage adaptive estimator overtakes it only at $n \approx 32{,}000$ for $n_{\mathrm{sel}}=1$, and not within the grid for $n_{\mathrm{sel}}=2$. With the centering of \Cref{rem:centering}, the crossover moves to $n \approx 500$ ($n_{\mathrm{sel}}=1$) and $n \approx 1{,}000$ ($n_{\mathrm{sel}}=2$). At $n = 4{,}000$ the centered adaptive rule beats not only the estimated non-adaptive rule but also the best fixed subset that exists in this world ($3.826$ versus $4.622$, $n_{\mathrm{sel}}=2$). Thus, in this data-generating process, the centered estimator is not underpowered at a sample size comparable to ours, and a fixed subset is not generically sufficient: when regime structure is present, the adaptive machinery can find it. The simulation is stylized ($J = 8$, a regime indicator recoverable from $A$, and a sizable adaptivity gap), so it indicates when adaptivity can pay rather than predicting its value on real data.

\begin{figure}[H]
\centering
\includegraphics[width=\linewidth]{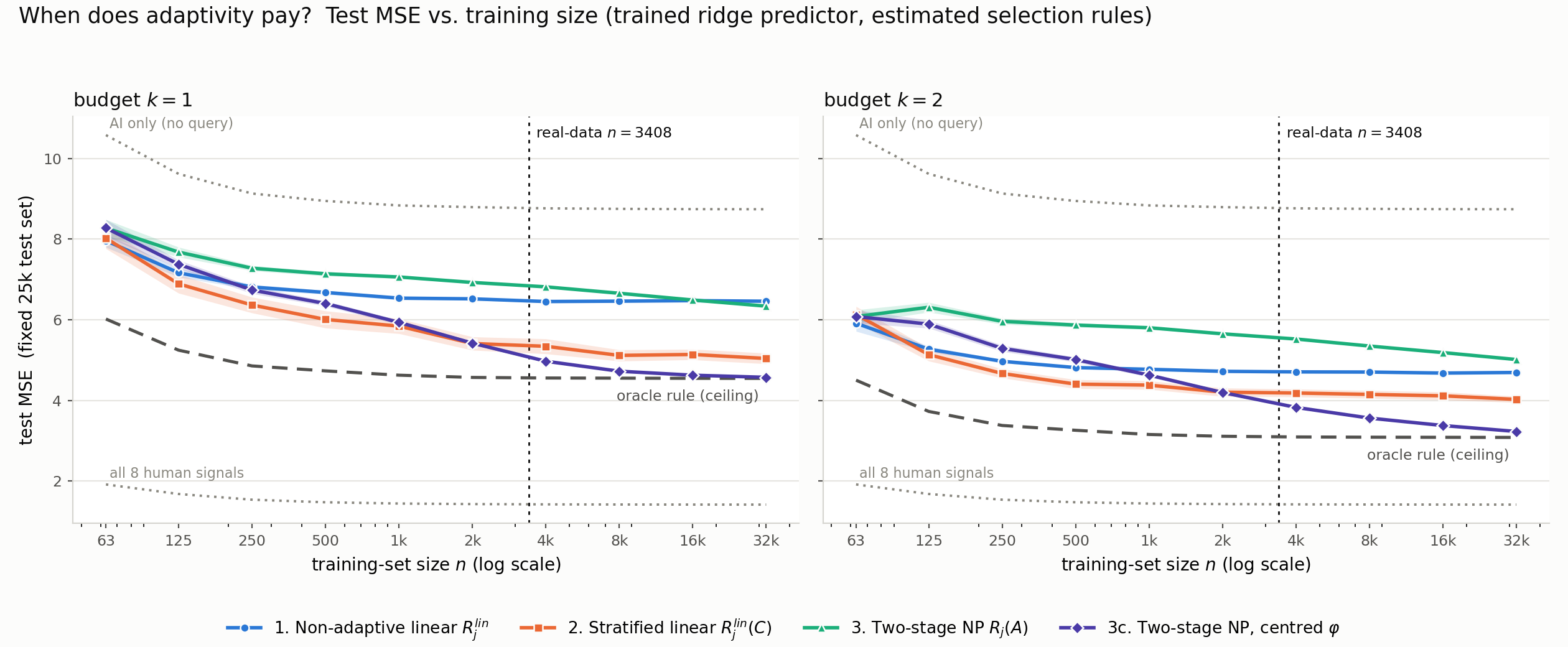}
\caption{Simulation: test MSE versus training-set size for the non-adaptive linear rule, the stratified linear rule, and the two-stage adaptive rule without and with centering. Dashed: oracle adaptive rule with the trained predictor; dotted: AI only and full human signal. The vertical line marks the size of the peer-review corpus ($n = 3{,}408$; each training split contains $2{,}727$ papers).}
\label{fig:e5}
\end{figure}

\begin{table}[H]
\centering
\scriptsize
\setlength{\tabcolsep}{3pt}
\caption{Simulation: test MSE (mean $\pm$ 95\% CI over 50 replicates) with a trained ridge predictor. Oracle rules are paired with the same trained predictor.}
\label{tab:e5}
\begin{adjustbox}{max width=\linewidth}
\begin{tabular}{@{}lcccccc@{}}
\toprule
 & \multicolumn{2}{c}{$n=500$} & \multicolumn{2}{c}{$n=4{,}000$} & \multicolumn{2}{c}{$n=32{,}000$} \\
\cmidrule(lr){2-3}\cmidrule(lr){4-5}\cmidrule(lr){6-7}
Rule & $k=1$ & $k=2$ & $k=1$ & $k=2$ & $k=1$ & $k=2$ \\
\midrule
AI only (no query) & $8.946\pm0.026$ & $8.946\pm0.026$ & $8.764\pm0.003$ & $8.764\pm0.003$ & $8.740\pm0.001$ & $8.740\pm0.001$ \\
Best fixed subset (oracle rule) & $6.493\pm0.018$ & $4.731\pm0.015$ & $6.350\pm0.002$ & $4.622\pm0.002$ & $6.334\pm0.000$ & $4.609\pm0.000$ \\
Estimated non-adaptive $R_j^{\mathrm{lin}}$ & $6.677\pm0.049$ & $4.813\pm0.032$ & $6.452\pm0.031$ & $4.707\pm0.017$ & $6.459\pm0.020$ & $4.692\pm0.000$ \\
Estimated stratified $R_j^{\mathrm{lin}}(C)$ & $6.008\pm0.212$ & $4.403\pm0.112$ & $5.342\pm0.186$ & $4.183\pm0.097$ & $5.040\pm0.133$ & $4.025\pm0.077$ \\
Estimated adaptive $R_j(A)$, original & $7.140\pm0.053$ & $5.868\pm0.056$ & $6.816\pm0.023$ & $5.522\pm0.024$ & $6.337\pm0.017$ & $5.013\pm0.015$ \\
Estimated adaptive $R_j(A)$, centered & $6.403\pm0.064$ & $5.013\pm0.062$ & $4.971\pm0.025$ & $3.826\pm0.024$ & $4.572\pm0.004$ & $3.229\pm0.011$ \\
Best adaptive rule (oracle) & $4.732\pm0.031$ & $3.259\pm0.024$ & $4.555\pm0.009$ & $3.093\pm0.007$ & $4.546\pm0.003$ & $3.082\pm0.002$ \\
Full human (all $J=8$) & $1.473\pm0.005$ & $1.473\pm0.005$ & $1.421\pm0.001$ & $1.421\pm0.001$ & $1.415\pm0.000$ & $1.415\pm0.000$ \\
\bottomrule
\end{tabular}

\end{adjustbox}
\end{table}

\subsection{Applying the framework to a new domain}
\label{app:porting}

\Cref{thm:general_reduction} and the two-stage estimator with its rate guarantee (\Cref{thm:orth_rate_main}) are domain-agnostic. The closed-form linear rule is one instantiation, which relies on \Cref{assump:nonadaptive,assump:linear}, and its adequacy can be checked in-domain. Run the nonparametric rules alongside the linear one, and compare linear and nonlinear prediction heads. In our data the non-adaptive nonparametric rule $R_j$ stays within $0.02$ MAE of the linear rules, and ridge and gradient boosting perform similarly, which indicates approximately linear structure. A large gap in a new domain would indicate the opposite, in which case the general rule applies at the higher sample cost quantified by \Cref{thm:orth_rate_main}. A coarse stratified rule (\Cref{app:stratified}) is an intermediate option when a small number of interpretable regimes can be identified from the AI signal. In domains where factors map to physically separable evidence (e.g., laboratory results versus imaging in a medical record), the per-factor cost model fits more literally than in peer review.

\subsection{Prompts}\label{appdx:prompts}

The user prompt for the LLM review task (generation of $A$ features) is in Figure \ref{fig:llm_review_prompt}.
The user prompt for extracting $H$ features from the review text is in Figure \ref{fig:human_review_prompt}.
The user prompt for generating agreement scores is in Figure \ref{fig:agreement_prompt}.
System prompts for all three tasks are in Figure \ref{fig:system_prompts}.

\begin{figure}[htbp]
    \centering
    \begin{tcolorbox}[
        width=0.92\linewidth,
        colback=gray!5,
        colframe=gray!50,
        boxrule=0.5pt,
        arc=2pt,
        left=6pt,
        right=6pt,
        top=6pt,
        bottom=6pt
    ]
    \scriptsize %\ttfamily
TITLE: \{\{TITLE\}\}

YEAR: \{\{YEAR\}\}

ABSTRACT\_OR\_TEXT:

\{\{TEXT\}\} \\

ASPECTS (name: definition):
\{\{ASPECTS\_BLOCK\}\} \\

LOGIC\_DICTIONARIES (per aspect, choose exactly one option for each):
\{\{LOGIC\_BLOCK\}\} \\

REFERENCE\_CUTOFF\_YEAR: \{\{REF\_YEAR\}\}. The paper was written in \{\{REF\_YEAR\}\}, so you can only use references on or before this year. \\

SCORING\_RUBRIC: 0.00=strongly negative; 0.25=weak; 0.50=mixed; 0.75=positive; 1.00=strongly positive. \\

TASK:

Work aspect by aspect, in the given order. For EACH aspect listed, produce exactly one assessment with:
  - score: float in [0,1]
  - logic\_code: the chosen logic option code
  - detailed: $\leq$200 words, grounded in the paper text and the reference year
After assessing all of the aspects, do the following:
   - Give the paper an overall score on the scale of 1 to 10 (based on how strongly you believe the paper should be accepted or rejected).
  - Provide a Y (integer): 1 if you recommend acceptance, 0 if you recommend rejection. \\

For the overall score, you can use the following anchor points (but your score can fall in between anchor points): \\

OVERALL SCORE ANCHORS:
- 1: Trivial or wrong
- 2: Strong rejection
- 3: Clear rejection
- 4: Ok but not good enough - rejection
- 5: Marginally below acceptance threshold
- 6: Marginally above acceptance threshold
- 7: Good paper, accept
- 8: Top 50% of accepted papers, clear accept
- 9: Top 15% of accepted papers, strong accept
- 10: Top 5% of accepted papers, seminal paper

RETURN ONLY THIS JSON SCHEMA:
\{
  "overall\_score": <float>,
  "Y": <1 for accept, 0 for reject>,
  "<AspectName1>": \{
      "score": <float>,
      "logic\_code": "<CODE>",
      "detailed": "<string>"
    \}, 
    "<AspectName2>": \{
      "score": <float>,
      "logic\_code": "<CODE>",
      "detailed": "<string>"
    \}, ...(one entry per aspect)
\}

IMPORTANT: within the "detailed" field, and ONLY within the "detailed" field, ONLY use SINGLE QUOTES or ESCAPED DOUBLE QUOTES \" if you must use quotation marks in your explanation.
    \end{tcolorbox}
    \caption{Prompt for collecting LLM reviews}
    \label{fig:llm_review_prompt}
\end{figure}

\begin{figure}[htbp]
    \centering
    \begin{tcolorbox}[
        width=0.92\linewidth,
        colback=gray!5,
        colframe=gray!50,
        boxrule=0.5pt,
        arc=2pt,
        left=6pt,
        right=6pt,
        top=6pt,
        bottom=6pt
    ]
    \scriptsize 
    
TITLE: \{\{TITLE\}\}

YEAR: \{\{YEAR\}\} \\

ASPECTS (name: definition):

\{\{ASPECTS\_BLOCK\}\} \\

LOGIC\_DICTIONARIES (per aspect; choose exactly one code):

\{\{LOGIC\_BLOCK\}\} \\

REVIEW\_TEXT: \{\{REVIEW\_TEXT\}\}

REVIEWER: \{\{REVIEW\_ID\}\} \\

Scoring guidance for a.score (when aspect is discussed):
Use any real value in [0,1]. Anchors (for intuition, not discretization):
0.00 strongly negative; 0.25 weak; 0.50 mixed/uncertain; 0.75 positive; 1.00 strongly positive.

Task (for each aspect in the ASPECTS list, in the given order):

	1.	Find sentences that most directly discuss that aspect.
    
		- If the aspect is NOT discussed: set s0 = 0, set score = 0.50 (neutral default), and set logic\_code = "NOT\_MENTIONED".
        
		- If discussed: set s0 = 1. \\
        
	2.	Provide the assessment using only REVIEW\_TEXT.
    
		- Set score to any float in [0,1] per the guidance.
        
		- Set logic\_code to one code from the logic dictionary for that aspect (do not invent new codes). \\

Return ONLY this JSON schema (no extra fields):
\{
	"<AspectName1>": \{
		"s0": 0 | 1,
		"score": <float in [0, 1]>,
		"logic\_code": "<[CODE]>"
	\},
	"<AspectName2>": \{
		"s0": 0 | 1,
		"score": <float in [0, 1]>,
		"logic\_code": "<[CODE]>"
	\}, ...(one entry per aspect)
\}

    \end{tcolorbox}
    \caption{Prompt for tagging the human reviews to generate human features $H$.}
    \label{fig:human_review_prompt}
\end{figure}

\begin{figure}[htbp]
    \centering
    \begin{tcolorbox}[
        width=0.92\linewidth,
        colback=gray!5,
        colframe=gray!50,
        boxrule=0.5pt,
        arc=2pt,
        left=6pt,
        right=6pt,
        top=6pt,
        bottom=6pt
    ]
    \scriptsize
PAPER:

TITLE: \{\{TITLE\}\}

YEAR: \{\{YEAR\}\} \\

LLM\_ASPECT\_INPUT (array):

\{\{LLM\_REVIEWS\}\}

// Each item: \{"aspect": "<name>", "llm\_answer": "<logic\_sentence>$\backslash$n $\backslash$n<detailed>"\} \\

REVIEW\_TEXT: \{\{REVIEW\_TEXT\}\}

REVIEWER: \{\{REVIEW\_ID\}\} \\

Scoring anchors for s1 (when aspect is mentioned):

- 1.00: Explicit agreement; human text clearly supports key points in llm\_answer.

- 0.75: Mostly agrees; minor caveats.

- 0.50: Mixed/unclear; both supportive and conflicting cues, or too vague.

- 0.25: Mostly disagrees; major contradictions.

- 0.00: Clear disagreement. \\

Task (for each item in LLM\_ASPECT\_INPUT, in order): \\

1. Find sentences that most directly discuss that aspect in REVIEW\_TEXT.

   - If NOT discussed: s0 = 0; s1 = 0.50 (neutral default).

   - If discussed: s0 = 1 and continue.

2. If the aspect was discussed, assess agreement using only REVIEW\_TEXT.

   - Set s1 via anchors. \\

Strict output schema: \\

\{

  "AspectName1": \{

    "s0": 0 | 1,

    "s1": <float>,

  \},

  "AspectName2": \{

    "s0": 0 | 1,

    "s1": <float>,

  \}, ...(one entry per aspect)
  
\}

    \end{tcolorbox}
    \caption{Prompt for generating agreement scores between the LLM and human reviews.}
    \label{fig:agreement_prompt}
\end{figure}

\begin{figure}[htbp]
    \centering
    \begin{tcolorbox}[
        width=0.92\linewidth,
        colback=gray!5,
        colframe=gray!50,
        boxrule=0.5pt,
        arc=2pt,
        left=6pt,
        right=6pt,
        top=6pt,
        bottom=6pt
    ]
    \scriptsize

    \textbf{LLM review}: You are a careful meta-review assistant. Base judgments ONLY on the provided paper text. 
If evidence is missing, say so and reflect uncertainty in the score. 
Return ONLY JSON. No markdown fences, no extra text. \\

\textbf{Human review}: You are a human-review proxy rater. Judge ONLY according to the REVIEW\_TEXT.
Do not infer beyond what the reviewer actually wrote. If an aspect is not discussed or evidence is ambiguous, say so and follow the scoring rules.
Output STRICT JSON; no extra keys or prose. \\

\textbf{Agreement}: 
You are an agreement adjudicator. Judge ONLY according to the HUMAN\_REVIEWS text. 
Do not infer beyond what the human reviewers actually wrote. If an aspect is not discussed or evidence is ambiguous, say so and follow the scoring rules. 
Output STRICT JSON; no extra keys or prose.
    
    \end{tcolorbox}
    \caption{System prompts for all tasks.}
    \label{fig:system_prompts}
\end{figure}

\end{document}